\documentclass[twoside]{article}

\usepackage[accepted]{aistats2020}

\usepackage{subcaption}
\usepackage[utf8]{inputenc}
\usepackage[T1]{fontenc}
\usepackage{hyperref}
\usepackage{url}
\usepackage{booktabs}
\usepackage{amsfonts}
\usepackage{nicefrac}
\usepackage{microtype}
\usepackage{amsmath}

\usepackage{graphicx}

\usepackage[round]{natbib}

\newcommand{\KL}{\mathcal{KL}}
\newcommand{\KLD}[2]{\KL\left(#1 \,\|\, #2 \right)}
\newcommand{\Lo}{\mathcal{L}}
\newcommand{\Lt}{\mathcal{L}_{\tau}}
\newcommand{\Ls}{\mathcal{L}_{*}}
\newcommand{\qq}{q_{\phi}}
\newcommand{\pp}{p_{\theta}}
\newcommand{\N}{\mathcal{N}}
\newcommand{\Ind}[1]{\mathbb{I}\left[#1\right]}

\usepackage{multirow}

\usepackage{amsthm}
\newtheorem{theorem}{Theorem}
\newtheorem{lemma}{Lemma}
\newtheorem{proposition}{Proposition}

\DeclareMathOperator*{\argmax}{arg\,max}
\DeclareMathOperator*{\Argmax}{Arg\,max}

\begin{document}

\twocolumn[

\aistatstitle{Deterministic Decoding for Discrete Data in Variational Autoencoders}

\aistatsauthor{Daniil Polykovskiy \And Dmitry Vetrov}

\aistatsaddress{Insilico Medicine \And  National Research University Higher School of Economics}
]

\begin{abstract}
Variational autoencoders are prominent generative models for modeling discrete data. However, with flexible decoders, they tend to ignore the latent codes.  In this paper, we study a VAE model with a deterministic decoder (DD-VAE) for sequential data that selects the highest-scoring tokens instead of sampling. Deterministic decoding solely relies on latent codes as the only way to produce diverse objects, which improves the structure of the learned manifold. To implement DD-VAE, we propose a new class of bounded support proposal distributions and derive Kullback-Leibler divergence for Gaussian and uniform priors. We also study a continuous relaxation of deterministic decoding objective function and analyze the relation of reconstruction accuracy and relaxation parameters. We demonstrate the performance of DD-VAE on multiple datasets, including molecular generation and optimization problems.
\end{abstract}

\section{Introduction}
Variational autoencoder \citep{Kingma2013} is an autoencoder-based generative model that provides high-quality samples in many data domains, including image generation \citep{razavi2019generating}, natural language processing \citep{semeniuta-etal-2017-hybrid}, audio synthesis \citep{hsu2018hierarchical}, and drug discovery \citep{zhavoronkov2019deep}.

Variational autoencoders use stochastic encoder and decoder. An encoder maps an object $x$ onto a distribution of the latent codes $\qq(z \mid x)$, and a decoder produces a distribution $\pp(x \mid z)$ of objects that correspond to a given latent code. In this paper, we analyze the impact of stochastic decoding on VAE models for discrete data and propose deterministic decoders as an alternative.

With complex stochastic decoders, such as PixelRNN \citep{pmlr-v48-oord16}, VAEs tend to ignore the latent codes, since the decoder is flexible enough to produce the whole data distribution $p(x)$ without using latent codes at all. Such behavior can damage the representation learning capabilities of VAE: we will not be able to use its latent codes for downstream tasks. A deterministic decoder, on the contrary, maps each latent code to a single data point, making it harder to ignore the latent codes, as they are the only source of variation.

\begin{figure}[t]
\begin{center}
\includegraphics[width=0.9\columnwidth]{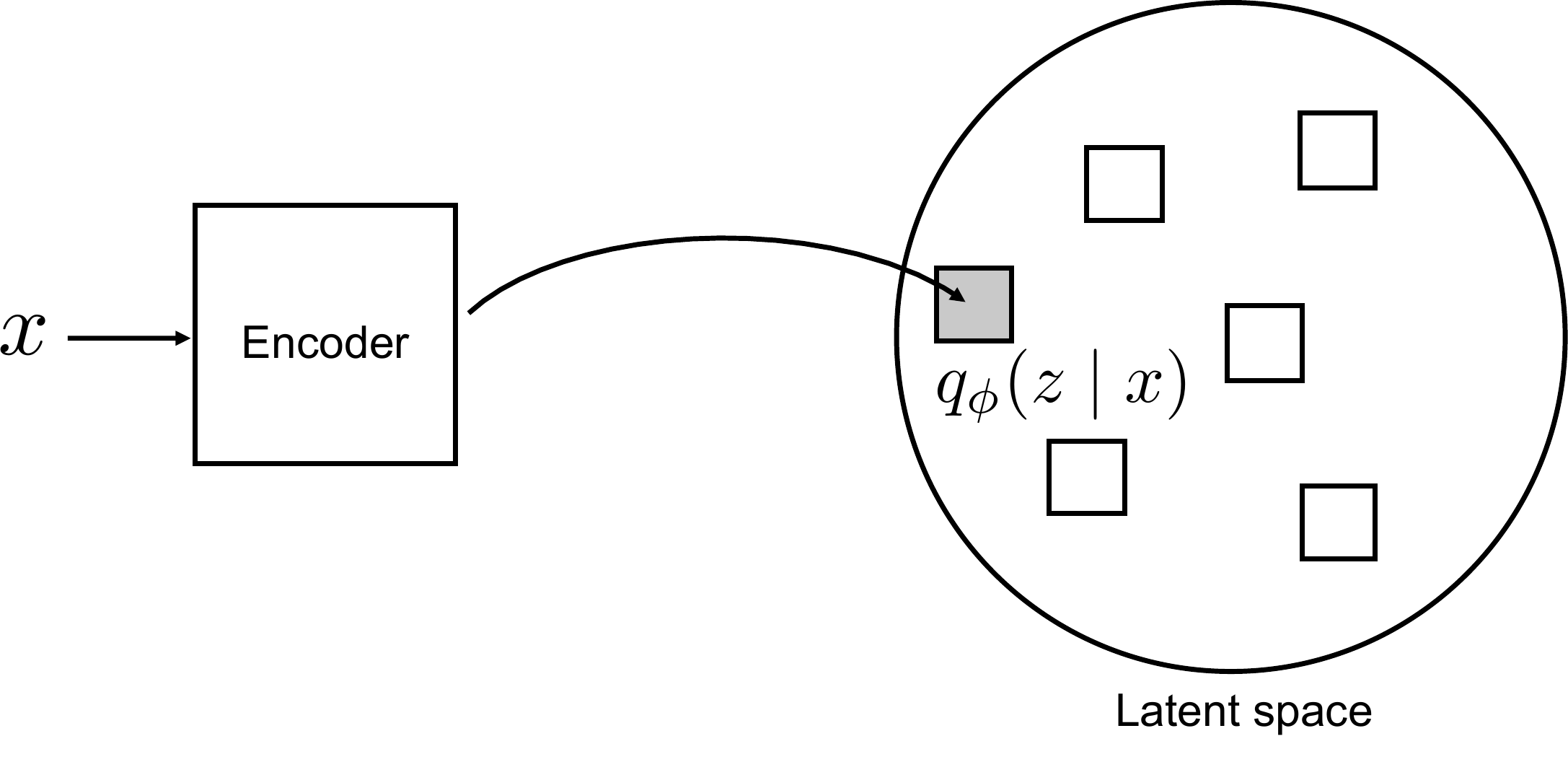}
\caption{The encoder of DD-VAE outputs parameters of bounded support distribution. With Gaussian proposals, lossless auto-encoding is impossible, since the proposals of any two objects overlap.} \label{fig:encoder}
\end{center}
\end{figure}

\begin{figure}[t]
\begin{center}
\includegraphics[width=0.8\columnwidth]{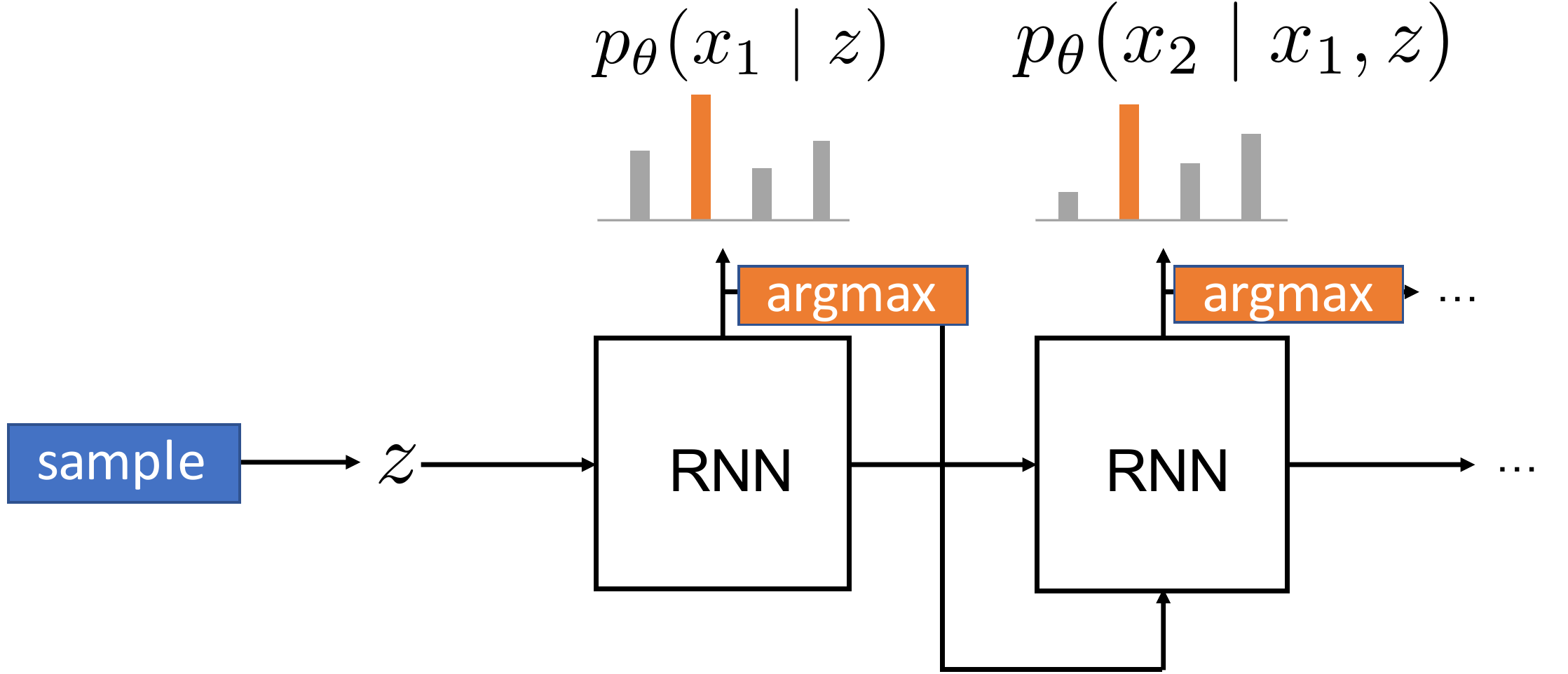}
\caption{During sampling, the decoder selects $\argmax$ of scores $\pp(x_i \mid x_{<i}, z)$. Hence, the only source of variation for the decoder is $z$. We propose a relaxed objective function to optimize through $\argmax$.} \label{fig:argmax}
\end{center}
\end{figure}

One application of latent codes of VAEs is Bayesian optimization of molecular properties. \citet{gomez2018automatic} trained a Gaussian process regressor on the latent codes of VAE and optimized the latent codes to discover molecular structures with desirable properties. With stochastic decoding, a Gaussian process has to account for stochasticity in target variables, since every latent code corresponds to multiple molecular structures. Deterministic decoding, on the other hand, simplifies the regression task, leading to better predictive quality, as we show in the experiments.

Our contribution is three-fold:
\begin{itemize}
    \item We formulate a model of deterministic decoder VAE (DD-VAE), derive its evidence lower bound and propose a convenient approximation with proven convergence to optimal parameters of non-relaxed objective;
    \item We show that lossless auto-encoding is impossible with full support proposal distributions and introduce bounded support distributions as a solution;
    \item We provide experiments on multiple datasets (synthetic, MNIST, MOSES, ZINC) to show that DD-VAE yields both a proper generative distribution and useful latent codes.
\end{itemize}
The code for reproducing the experiments is available at \url{https://github.com/insilicomedicine/DD-VAE}.
\section{Deterministic Decoder VAE (DD-VAE)}
In this section, we formulate a deterministic decoder variational autoencoder (DD-VAE). Next, we show the need for bounded support proposals and introduce them in Section~\ref{sec:bounded_support}. In Section~\ref{sec:relaxation}, we propose a continuous relaxation of the DD-VAE's ELBO. In Section~\ref{sec:relation}, we prove that the optimal solution of the relaxed problem matches the optimal solution of the original problem.

Variational autoencoder (VAE) consists of an encoder $\qq(z \mid x)$ and a decoder $\pp(x \mid z)$. The model learns a mapping of data distribution $p(x)$ onto a prior distribution of latent codes $p(z)$ which is often a standard Gaussian $\N(0, I)$. Parameters $\theta$ and $\phi$ are learned by maximizing a lower bound $\Lo(\theta, \phi)$ on log marginal likelihood $\log p(x)$. $\Lo(\theta, \phi)$ is known as an evidence lower bound (ELBO):
\begin{equation}
\begin{split}
    \Lo(\theta, \phi) =  \mathbb{E}_{x \sim p(x)}\big[& \mathbb{E}_{z \sim \qq(z \mid x)}\log \pp(x \mid z) \\
    & - \KLD{\qq(z \mid x)}{p(z)}\big]. \label{eq:elbo_classic}
\end{split}
\end{equation}
The $\log \pp(x \mid z)$ term in Eq.~\ref{eq:elbo_classic} is a reconstruction loss, and the $\KL$ term is a Kullback-Leibler divergence that encourages latent codes to be marginally distributed as $p(z)$.

For sequence models, $x$ is a sequence $x_1, x_2, \dots, x_{|x|}$, where each token of the sequence is an element of a finite vocabulary $V$, and $|x|$ is the length of sequence $x$. A decoding distribution for sequences is often parameterized as a recurrent neural network that produces a probability distribution over each token $x_{i}$ given the latent code and all previous tokens. The ELBO for such model is:
\begin{equation}
\begin{split}
    \Lo(\theta, \phi) =  \mathbb{E}_{x \sim p(x)}\bigg[& \mathbb{E}_{z \sim \qq(z \mid x)}\sum_{i=1}^{|x|}\log \pi^{\theta}_{x, i, x_i}(z) \\ & -\KLD{\qq(z \mid x)}{p(z)}\bigg], \label{eq:elbo_sequence}
\end{split}
\end{equation}
where $\pi^{\theta}_{x, i, s}(z) = \pp(x_i=s \mid z, x_1, x_2, \dots, x_{i-1})$.

In deterministic decoders, we decode a sequence $\widetilde{x}_{\theta}(z)$ from a latent code $z$ by taking a token with the highest score at each iteration:
\begin{equation}
\widetilde{x}_i = \argmax_{s \in V}\pp(s \mid z, {x}_1, \dots, {x}_{i-1}) = \argmax_{s \in V}\pi^{\theta}_{x, i, s}(z)
\label{eq:argmax}
\end{equation}
To avoid ambiguity, when two tokens have the same maximal probability, $\argmax$ is equal to a special ``undefined'' token that does not appear in the data. Such formulation simplifies derivations in the remaining of the paper. We also assume $\pi^{\theta}_{x, i, s} \in [0, 1]$ for convenience. 
After decoding $\widetilde{x}_{\theta}$, reconstruction term of ELBO is an indicator function which is one, if the model reconstructed a correct sequence, and zero otherwise:
\begin{equation}
    p\left(x \mid \widetilde{x}_{\theta}\left(z\right)\right) = \begin{cases}1, & \widetilde{x}_{\theta}(z) = x \\ 0, & \mathrm{otherwise}\end{cases} \label{eq:delta}
\end{equation}
\begin{equation}
\begin{split}
\Ls(\theta, \phi) = &  \mathbb{E}_{x \sim p(x)}\big[\mathbb{E}_{z \sim \qq(z \mid x)}\log p\left(x \mid \widetilde{x}_{\theta}(z)\right)\\
& -\KLD{\qq(z \mid x)}{p(z)}\big]. \label{eq:elbo0}
\end{split}
\end{equation}
The $\Ls(\theta, \phi)$ is $-\infty$ if the model has non-zero reconstruction error rate, leading us to two questions: is $\Ls$ finite for some parameters $(\theta, \phi)$ and how to optimize $\Ls$. We answer both questions in the following sections.

\subsection{Proposal distributions with bounded support \label{sec:bounded_support}}
In this section, we discuss bounded support proposal distributions $\qq(z \mid x)$ in VAEs and why they are crucial for deterministic decoders.

Variational Autoencoders often use Gaussian proposal distributions 
\begin{equation}
    \qq(z \mid x) = \mathcal{N}\left(z \mid \mu_{\phi}\left(x\right), \Sigma_{\phi}\left(x\right)\right),
\end{equation}
where $\mu_{\phi}(x)$ and $\Sigma_{\phi}(x)$ are neural networks modeling the mean and the covariance matrix of the proposal distribution. For a fixed $z$, Gaussian density $\qq(z \mid x)$ is positive for any $x$. Hence, a lossless decoder has to decode every $x$ from every $z$ with a positive probability. However, a deterministic decoder can produce only a single data point $\widetilde{x}_{\theta}(z)$ for a given $z$, making reconstruction term of $\Ls$ minus infinity. To avoid this problem, we propose to use bounded support proposal distributions.

As bounded support proposal distributions, we suggest to use factorized distributions with marginals defined using a kernel $K$:
\begin{equation}
    \qq(z \mid x) = \prod_{i=1}^d \frac{1}{\sigma^{\phi}_i(x)}K\left(\frac{z_i-\mu^{\phi}_i(x)}{\sigma^{\phi}_i(x)}\right),
\end{equation}
where $\mu^{\phi}_i(x)$ and $\sigma^{\phi}_i(x)$ are neural networks that model location and bandwidth of a kernel $K$; the support of $i$-th dimension of $z$ in $\qq(z \mid x)$ is a range $\left[\mu^{\phi}_i(x)-\sigma^{\phi}_i(x),\, \mu^{\phi}_i(x)+\sigma^{\phi}_i(x)\right]$. We choose a kernel such that we can compute $\KL$ divergence between $q(z \mid x)$ and a prior $p(z)$ analytically. If $p(z)$ is factorized, $\KL$ divergence is a sum of one-dimensional $\KL$ divergences:
\begin{equation}
    \KLD{\qq(z \mid x)}{p(z)} = \sum_{i=1}^d \KLD{\qq(z_i \mid x)}{p(z_i)}.
\end{equation}
In Table \ref{tab:kl_finite}, we show $\KL$ divergence for some bounded support kernels and illustrate their densities in Figure~\ref{fig:kernel}. Note that the form of $\KL$ divergence is very similar to the one for a Gaussian proposal distribution---they only differ in a constant multiplier for $\sigma^2$ and an additive constant. For sampling, we use rejection sampling from $K(\epsilon)$ with a uniform proposal $K(0) \cdot \mathcal{U}[-1, 1]$ and apply a reparametrization trick to obtain a final sample: $z = \epsilon \cdot \sigma + \mu$. The acceptance rate in such sampling is $\frac{1}{2K(0)}$. Hence, to sample a batch of size $N$, we sample $N \cdot 2K(0)$ objects and repeat sampling until we get at least $N$ accepted samples. We also store a buffer with excess samples and use them in the following batches.

\begin{figure}[h]
\begin{center}
\includegraphics[width=0.99\columnwidth]{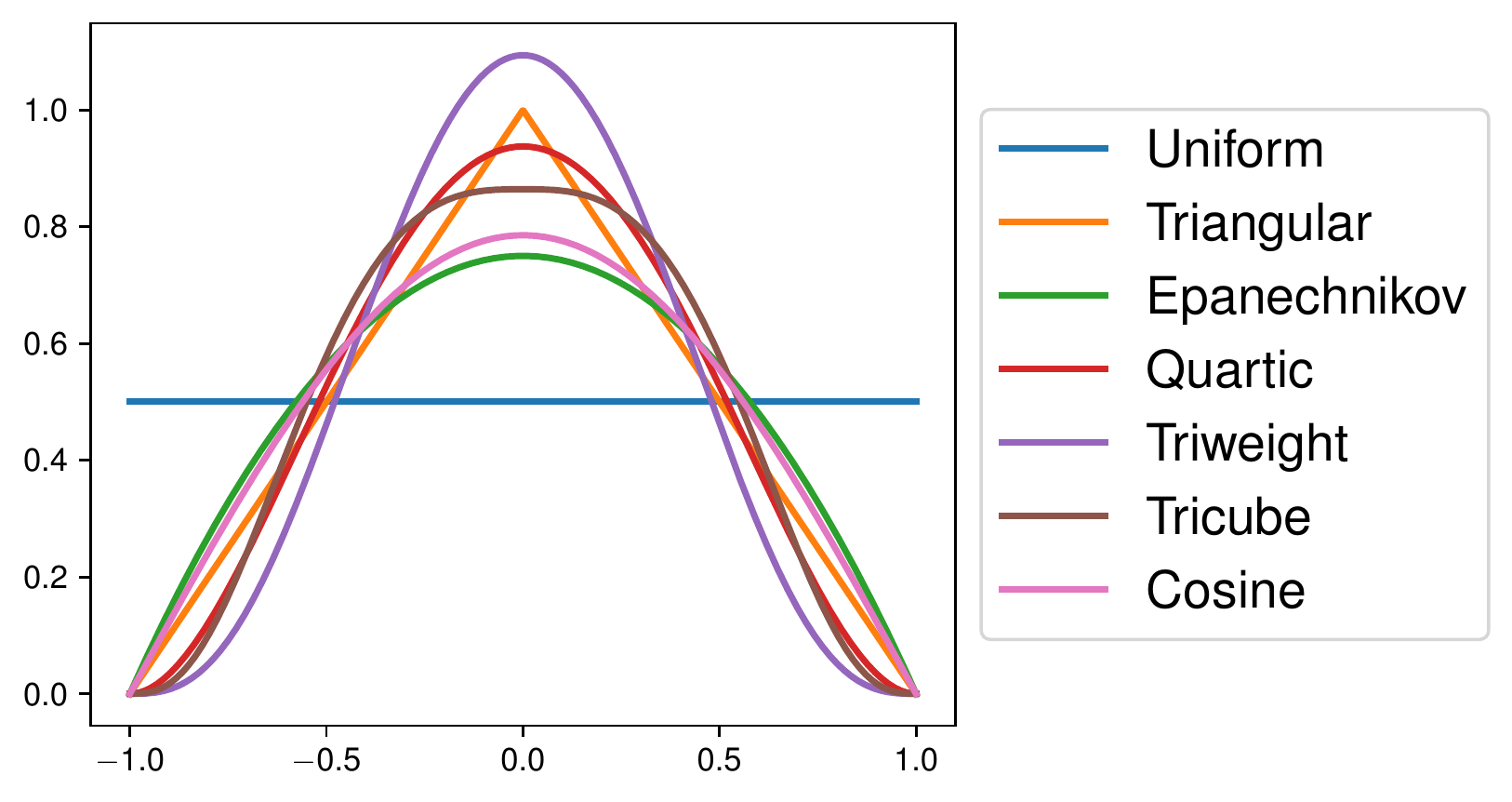}
\caption{Bounded support proposals with $\mu=0$ and $\sigma=1$ for which we derived $\KL$ divergence.}
\label{fig:kernel}
\end{center}
\end{figure}

With bounded support proposals, we can use a uniform distribution $U[-1, 1]^d$ as a prior in VAE as long as the support of $\qq(z \mid x)$ lies inside the support of a prior distribution. In practice, we ensure this by transforming $\mu$ and $\sigma$ from the encoder into $\mu'$ and $\sigma'$ using the following transformation:
\begin{eqnarray}
\mu' = \frac{\tanh(\mu + \sigma) + \tanh(\mu-\sigma)}{2}, \\
\sigma' = \frac{\tanh(\mu + \sigma) - \tanh(\mu-\sigma)}{2}.
\end{eqnarray}
We report derived $\KL$ divergences for a uniform prior in Table~\ref{tab:kl_uniform}.

\begin{table*}[ht]
\caption{$\mathcal{KL}$ divergence between a proposal $q$ with support $\left|z-\mu\right| \le \sigma$ and prior $\mathcal{N}(0, 1)$.}
\label{tab:kl_finite}
\begin{center}
\begin{small}
\begin{sc}
\begin{tabular}{lll}
\toprule
Kernel & $q(z \mid \sigma, \mu)$ & $\mathcal{KL}\left(q\left(z \mid \sigma, \mu\right) \|\, \mathcal{N}\left(0, 1\right)\right)$ \\
\midrule
Uniform & $\frac{1}{2 \sigma}$ & $\frac{1}{2}\mu^2 + \frac{1}{6}\sigma^2 - \log \sigma+ \frac{1}{2}\log(2\pi) - \log 2 $ \\
Triangular & $\frac{1}{\sigma}\left(1-\left.\left| \frac{z-\mu}{\sigma}\right.\right| \right)$ & $\frac{1}{2}\mu^2 + \frac{1}{12}\sigma^2 - \log \sigma+ \frac{1}{2} \log(2\pi) - \frac{1}{2} $ \\
Epanechnikov & $\frac{3}{4 \sigma} \left(1-\frac{(z-\mu)^2}{\sigma^2}\right)$ &  $\frac{1}{2}\mu^2+\frac{1}{10}\sigma^2-\log \sigma+\frac{1}{2} \log (2 \pi )-\frac{5}{3}+\log 3$ \\
Quartic & $\frac{15}{16 \sigma} \left(1-\frac{(z-\mu)^2}{\sigma^2}\right)^2$ & $\frac{1}{2}\mu^2+\frac{1}{14}\sigma^2-\log\sigma +\frac{1}{2} \log (2 \pi )-\frac{47}{15}+ \log{15}$ \\
Triweight & $\frac{35}{32 \sigma} \left(1-\frac{(z-\mu)^2}{\sigma^2}\right)^3$ & $\frac{1}{2}\mu^2+\frac{1}{18}\sigma^2-\log \sigma+\frac{1}{2} \log (2 \pi )-\frac{319}{70}+\log 70$ \\
Tricube & $\frac{70}{81\sigma} \left(1-\frac{\left|z-\mu\right|^3}{\sigma^3}\right)^3$ & $\frac{1}{2}\mu^2+\frac{35}{486}\sigma^2 - \log \sigma  + \frac{1}{2}\log \left(2 \pi\right)+\frac{\pi  \sqrt{3}}{2}-\frac{1111}{140} + \log{70\sqrt{3}}$ \\
Cosine & $\frac{\pi}{4 \sigma}  \cos \left(\frac{\pi  (z-\mu)}{2 \sigma}\right)$ & $\frac{1}{2}\mu^2 + \left(\frac{1}{2}-\frac{4}{\pi^2}\right) \sigma^2-\log\sigma + \frac{1}{2}\log \left(2 \pi\right)-1 + \log\frac{\pi}{2}$ \\
\midrule
Gaussian $(-\infty, \infty)$ & $\frac{1}{\sqrt{2\pi}\sigma}e^{-\frac{1}{2\sigma^2}(z-\mu)^2}$ & $\frac{1}{2}\mu^2 + \frac{1}{2}\sigma^2 - \log\sigma - \frac{1}{2}$ \\
\bottomrule
\end{tabular}
\end{sc}
\end{small}
\end{center}
\end{table*}

\begin{table}[h]
\caption{$\mathcal{KL}$ divergence between a proposal $q$ with support $\left|z-\mu\right| \le \sigma$ and a uniform prior $\mathcal{U}[-1, 1]$ if the support of $q$ lies in $[-1, 1]$.}
\label{tab:kl_uniform}
\begin{center}
\begin{small}
\begin{sc}
\begin{tabular}{lll}
\toprule
Kernel & $\mathcal{KL}\left(q\left(z \mid \sigma, \mu\right) \|\, \mathcal{U}\left[-1, 1\right]\right)$ \\
\midrule
Uniform & $-\log \sigma$ \\
Triangular & $-\frac12 + \log2 -\log \sigma$ \\
Epanechnikov & $-\frac53 + \log6 -\log \sigma$\\
Quartic & $-\frac{47}{15} + \log30 -\log \sigma$ \\
Triweight & $-\frac{319}{70} + \log140 -\log \sigma$ \\
Tricube & $-\frac{1111}{140} + \frac{\pi  \sqrt{3}}{2} + \log{140 \sqrt{3}} - \log\sigma$ \\
Cosine & $-1 + \log\pi -\log\sigma$ \\
\bottomrule
\end{tabular}
\end{sc}
\end{small}
\end{center}
\end{table}

For discrete data, with bounded support proposals we can ensure that for sufficiently flexible encoder and decoder, there exists a set of parameters $(\theta, \phi)$ for which proposals $\qq(z \mid x)$ do not overlap for different $x$, and hence ELBO $\Ls(\theta, \phi)$ is finite. For example, we can enumerate all objects and map $i$-th object to a range $[i, i+1]$.

\subsection{Approximating ELBO \label{sec:relaxation}}
In this section, we discuss how to optimize a discontinuous function $\Ls(\theta, \phi)$ by approximating it with a smooth function. We also show the convergence of optimal parameters of an approximated ELBO to the optimal parameters of the original function in the next section.

We start by equivalently defining $\argmax$ from Eq.~\ref{eq:argmax} for some array $r$:
\begin{equation}
\begin{split}
\Ind{i = \argmax_{j}r_j} = \prod_{j \neq i}\Ind{r_i > r_j} \label{eq:argmax_ind}
\end{split}
\end{equation}
We approximate Eq.~\ref{eq:argmax_ind} by introducing a smooth relaxation $\sigma_{\tau}(x)$ of an indicator function $\Ind{x > 0}$ parameterized with a temperature parameter $\tau \in (0, 1)$:
\begin{equation}
    \Ind{x > 0} \approx \sigma_{\tau}(x) = \frac{1}{1 + \exp{(-x / \tau)}\left[\frac{1}{\tau} - 1\right]}\label{eq:approximate_ind}
\end{equation}

\begin{figure}[h]
\begin{center}
\includegraphics[width=0.9\columnwidth]{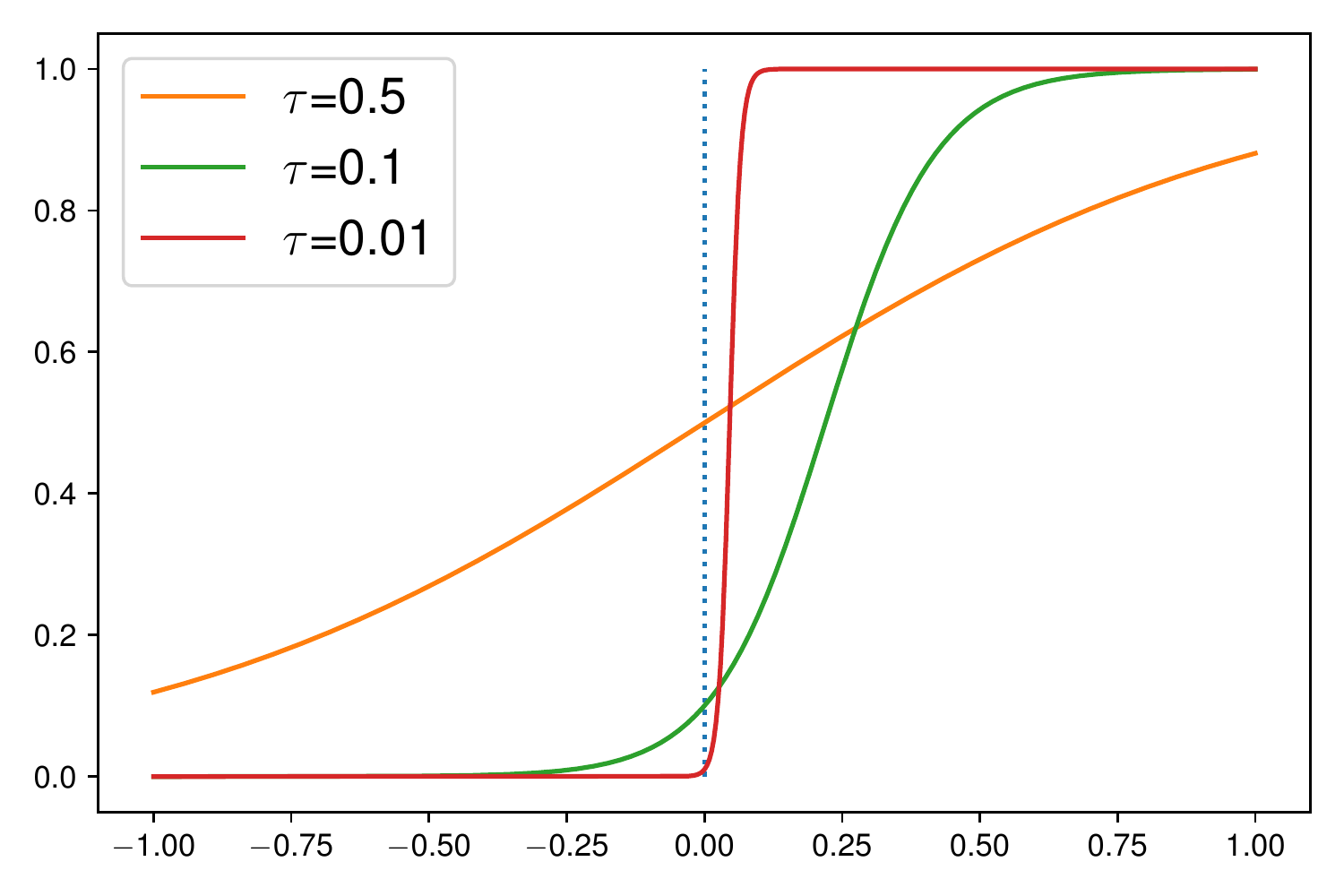}
\caption{Relaxation $\sigma_{\tau}(x)$ of an indicator function $\Ind{x > 0}$ for different $\tau$.}
\label{fig:temp_likelihood}
\end{center}
\end{figure}
Note that $\sigma_{\tau}(x)$ converges to $\Ind{x > 0}$ pointwise. In Figure~\ref{fig:temp_likelihood} we show function $\sigma_{\tau}(x)$ for different values of $\tau$. Substituting $\argmax$ with the proposed relaxation, we get the following approximation of the evidence lower bound:
\begin{equation}
\begin{split}
    \Lt(\theta,& \phi) =  \mathbb{E}_{x \sim p(x)}\bigg[\\
    & \mathbb{E}_{z \sim \qq(z \mid x)}\sum_{i=1}^{|x|}
    \sum_{s \neq x_i}\log\sigma_{\tau}\left(\pi^{\theta}_{x, i, x_i}(z) - \pi^{\theta}_{x, i, s}(z)\right)\\
    & -\KLD{\qq(z \mid x)}{p(z)}\bigg] \label{eq:temp_likelihood}
\end{split}
\end{equation}
A proposed $\Lt$ is finite for $0 < \tau < 1$ and converges to $\Ls$ pointwise. In the next section, we formulate the theorem that shows that if we gradually decrease temperature $\tau$ and solve maximization problem for ELBO $\Lt$, we will converge to optimal parameters of a non-relaxed ELBO $\Ls$.

\subsection{Convergence of optimal parameters of $\Lt$ to optimal parameters of $\Ls$ \label{sec:relation}}
In this section, we introduce auxiliary functions that are useful for assessing the quality of the model and formulate a theorem on the convergence of optimal parameters of $\Lt$ to optimal parameters of $\Ls$. 

Denote $\Delta(\widetilde{x}_{\theta}, \phi)$ a sequence-wise error rate for a given encoder and decoder:
\begin{equation}
    \Delta(\widetilde{x}_{\theta}, \phi) = \mathbb{E}_{x \sim p(x)}\mathbb{E}_{z \sim \qq(z \mid x)}\Ind{\widetilde{x}_{\theta}(z) \neq x} \label{eq:defdelta}.
\end{equation}
For a given $\phi$, we can find an optimal decoder and a corresponding sequence-wise error rate $\Delta(\phi)$ by rearranging the terms in Eq.~\ref{eq:defdelta} and applying importance sampling:
\begin{equation}
\begin{split}
    \Delta(\widetilde{x}_{\theta}, \phi) & =  1 - \mathbb{E}_{z \sim p(z)}\mathbb{E}_{x \sim p(x)}\frac{\qq(z \mid x)}{p(z)}\Ind{\widetilde{x}_{\theta}(z) = x} \\
     & = 1 - \mathbb{E}_{z \sim p(z)}\frac{p\big(\widetilde{x}_{\theta}\left(z\right)\big)\qq(z \mid \widetilde{x}_{\theta}\left(z\right))}{p(z)}  \\
     & \geq 1 - \mathbb{E}_{z \sim p(z)}\frac{p\big(\widetilde{x}^*_{\phi}\left(z\right)\big)\qq(z \mid \widetilde{x}^*_{\phi}\left(z\right))}{p(z)} \\
     & = \Delta(\widetilde{x}^*_{\phi}, \phi) = \Delta(\phi) \ge 0, \label{eq:def_delta_phi}
\end{split}
\end{equation}
where $\widetilde{x}^*_{\phi}\left(z\right)$ is an optimal decoder given by:
\begin{equation}
    \widetilde{x}^*_{\phi}\left(z\right) \in \Argmax_{x \in \chi}p(x)\qq(z \mid x). \label{eq:optimal_decoder}
\end{equation}
Here, $\chi$ is a set of all possible sequences. Denote $\Omega$ a set of parameters for which ELBO $\Ls$ is finite:
\begin{equation}
    \Omega = \left\{(\theta, \phi) \mid \Ls(\theta, \phi) > -\infty \right\}
\end{equation}
\begin{theorem}
\label{theorem:1}
Assume that $\Omega \neq \emptyset$, length of sequences in $\chi$ is bounded $(\exists L: |x| \le L, \forall x \in \chi)$, and $\Theta$ and $\Phi$ are compact sets of possible parameter values. Assume that $\qq(z \mid x)$ is equicontinuous in total variation for any $\phi$ and $x$:
\begin{equation}
\begin{split}
    & \forall\epsilon>0, \exists\delta=\delta(\epsilon, x, \phi)>0: \\
    & \|\phi-\phi'\|<\delta \Rightarrow \int\left|\qq(z \mid x)-q_{\phi'}(z \mid x)\right|dz < \epsilon.
\end{split}
\end{equation}
Let $\tau_n, \phi_n, \theta_n$ be such sequences that:
\begin{eqnarray}
\lim_{n \to \infty} \tau_n = 0, ~~\tau_n \in (0, 1),\\
(\theta_n, \phi_n) \in \Argmax_{\theta \in \Theta, \phi \in \Phi} \Lo_{\tau_n}(\theta, \phi),
\end{eqnarray}
sequence $\{\phi_n\}$ converges to $\widetilde{\phi}$, and for any $\phi$ such that $\Delta(\phi)=0$ exists $\theta$ such that $\Delta(\widetilde{x}_{\theta}, \phi) = 0$. Let $\widetilde{\theta}$ be:
\begin{equation}
    \widetilde{\theta} \in \Argmax_{\theta \in \Theta} \Ls(\theta, \widetilde{\phi}). \label{eq:find_optimal_decoder}
\end{equation}
Then the sequence-wise error rate decreases asymptotically as
\begin{equation}
    \Delta(\widetilde{x}_{\theta_n}, \phi_n) = \mathcal{O}\left(\frac{1}{\log(1/\tau_n)}\right),
\end{equation}
$\Delta(\widetilde{\phi})=0$, and final parameters $(\widetilde{\theta}, \widetilde{\phi})$ solve the optimization problem for $\Ls$: 
\begin{equation}
    \Ls(\widetilde{\theta}, \widetilde{\phi}) = \sup_{\theta \in \Theta, \phi \in \Phi}\Ls(\theta, \phi).
\end{equation}
\end{theorem}
\begin{proof}
See Appendix~\ref{sec:th1 proof}.
\end{proof}

The maximum length of sequences is bounded in the majority of practical applications. Equicontinuity assumption is satisfied for all distributions we considered in Table~\ref{tab:kl_finite} if $\mu$ and $\sigma$ depend continuously on $\phi$ for all $x \in \chi$. $\Omega$ is not empty for bounded support distributions when encoder and decoder are sufficiently flexible, as discussed in Section~\ref{sec:bounded_support}.

Eq.~\ref{eq:find_optimal_decoder} suggests that after we finish training the autoencoder, we should fix the encoder and fine-tune the decoder. Since $\Delta(\widetilde{\phi})=0$, the optimal stochastic decoder for such $\phi$ is deterministic---any $z$ corresponds to a single $x$ except for a zero probability subset. In theory, we could learn $\widetilde{\theta}$ for a fixed $\widetilde{\phi}$ by optimizing a reconstruction term of ELBO from Eq.~\ref{eq:elbo_sequence}:
\begin{equation}
    \Lo_{\textrm{rec}}(\theta) = \mathbb{E}_{x \sim p(x)}\mathbb{E}_{z \sim q_{\widetilde{\phi}}(z \mid x)}\sum_{i=1}^{|x|}\log \pi^{\theta}_{x, i, x_i},
\end{equation}
but since in practice we do not anneal the temperature exactly to zero, we found such fine-tuning optional.

\section{Related Work}
Autoencoder-based generative models consist of an encoder-decoder pair and a regularizer that forces encoder outputs to be marginally distributed as a prior distribution. This regularizer can take a form of a $\KL$ divergence as in Variational Autoencoders \citep{Kingma2013} or an adversarial loss as in Adversarial Autoencoders \citep{Makhzani2015} and Wasserstein Autoencoders \citep{tolstikhin2017wasserstein}. Besides autoencoder-based generative models, generative adversarial networks \citep{Goodfellow2014} and normalizing flows \citep{dinh2014nice, dinh2016density} were shown to be useful for sequence generation \citep{yu2017seqgan, pmlr-v80-oord18a}.

Variational autoencoders are prone to posterior collapse when the encoder outputs a prior distribution, and a decoder learns the whole distribution $p(x)$ by itself. Posterior collapse often occurs for VAEs with autoregressive decoders such as PixelRNN \citep{pmlr-v48-oord16}. Multiple approaches were proposed to tackle posterior collapse, including decreasing the weight $\beta$ of a $\KL$ divergence \citep{higgins2017beta}, or encouraging high mutual information between latent codes and corresponding objects \citep{zhao2019infovae}.

Other approaches modify a prior distribution, making it more complex than a proposal: a Gaussian mixture model \citep{pmlr-v84-tomczak18a, kuznetsov2019prior}, autoregressive priors \citep{chen2016variational}, or training a deterministic encoder and obtaining prior with a kernel density estimation \citep{ghosh2019variational}. Unlike these approaches, we conform to the standard Gaussian prior, and study the required properties of encoder and decoder to achieve deterministic decoding.

Deep generative models became a prominent approach in drug discovery as a way to rapidly discover potentially active molecules \citep{Polykovskiy2018, zhavoronkov2019deep}. Recent works explored feature-based \citep{kadurin2016cornucopia}, string-based \citep{gomez2018automatic, Segler2018-eo}, and graph-based \citep{pmlr-v80-jin18a, De_Cao2018-ju, You2018-yp} generative models for molecular structures. In this paper, we use a simplified molecular-input line-entry system (SMILES)  \citep{Weininger1970, Weininger1989} to represent the molecules---a system that represents a molecular graph as a string using a depth-first search order traversal. Multiple algorithms were proposed to exploit SMILES structure using formal grammars \citep{kusner2017grammar, dai2018syntax}.

\begin{figure*}[ht]
\begin{center}
\begin{subfigure}{0.3\textwidth}
\includegraphics[width=\linewidth]{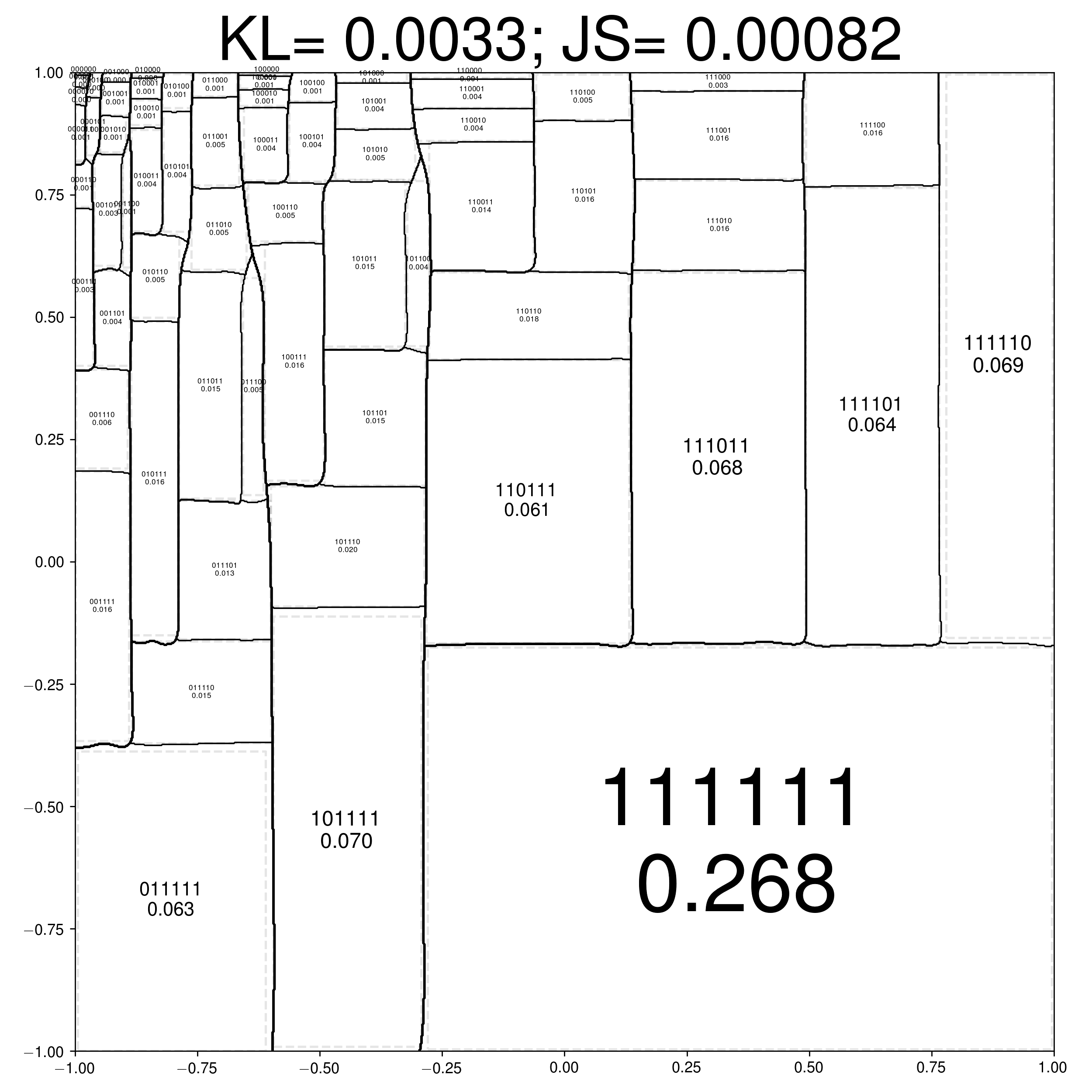} \caption{DD-VAE. U prior; U proposal} \label{fig:1a}
\end{subfigure}
\begin{subfigure}{0.3\textwidth}
\includegraphics[width=\linewidth]{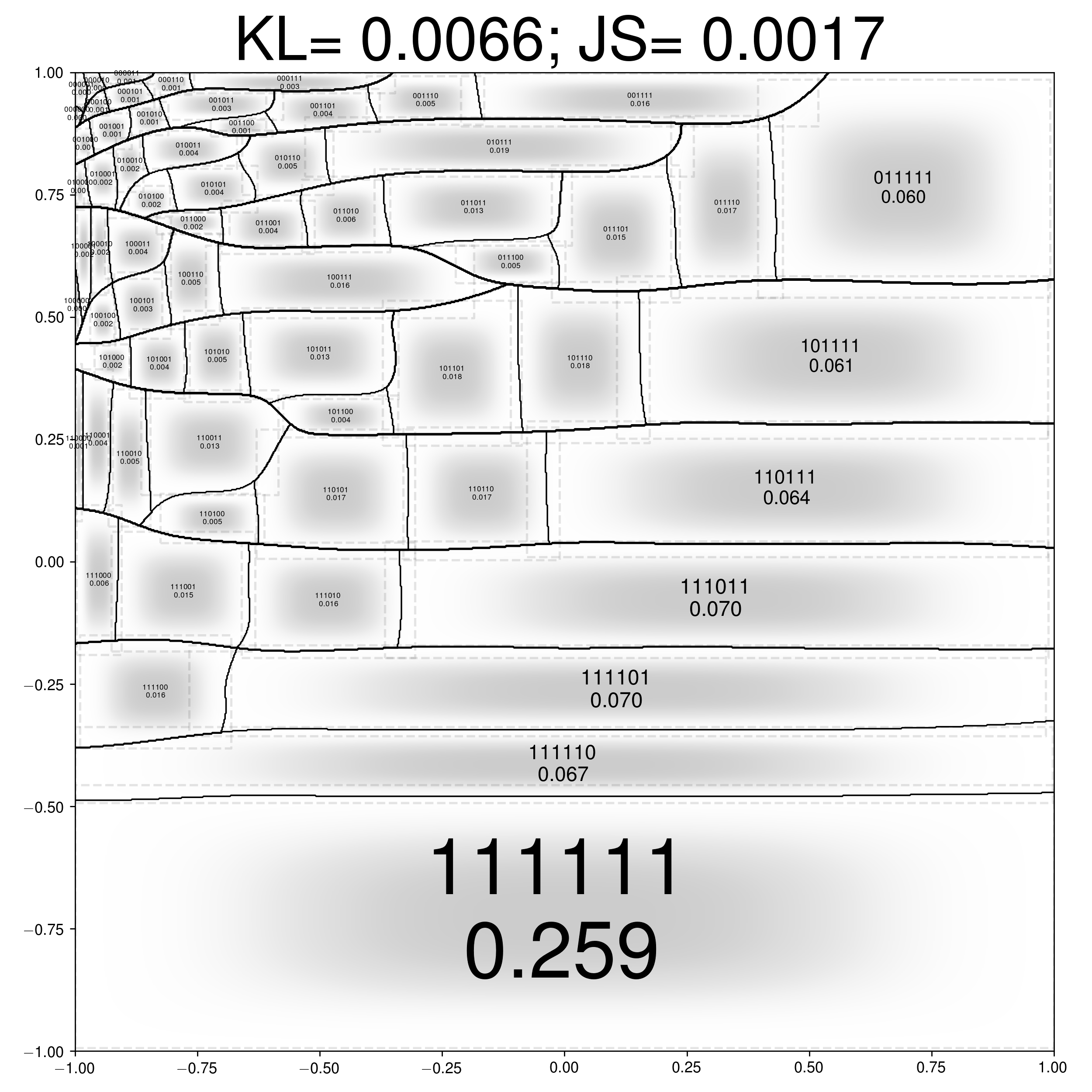} \caption{DD-VAE. U prior; T proposal} \label{fig:1a}
\end{subfigure}
\begin{subfigure}{0.3\textwidth}
\includegraphics[width=\linewidth]{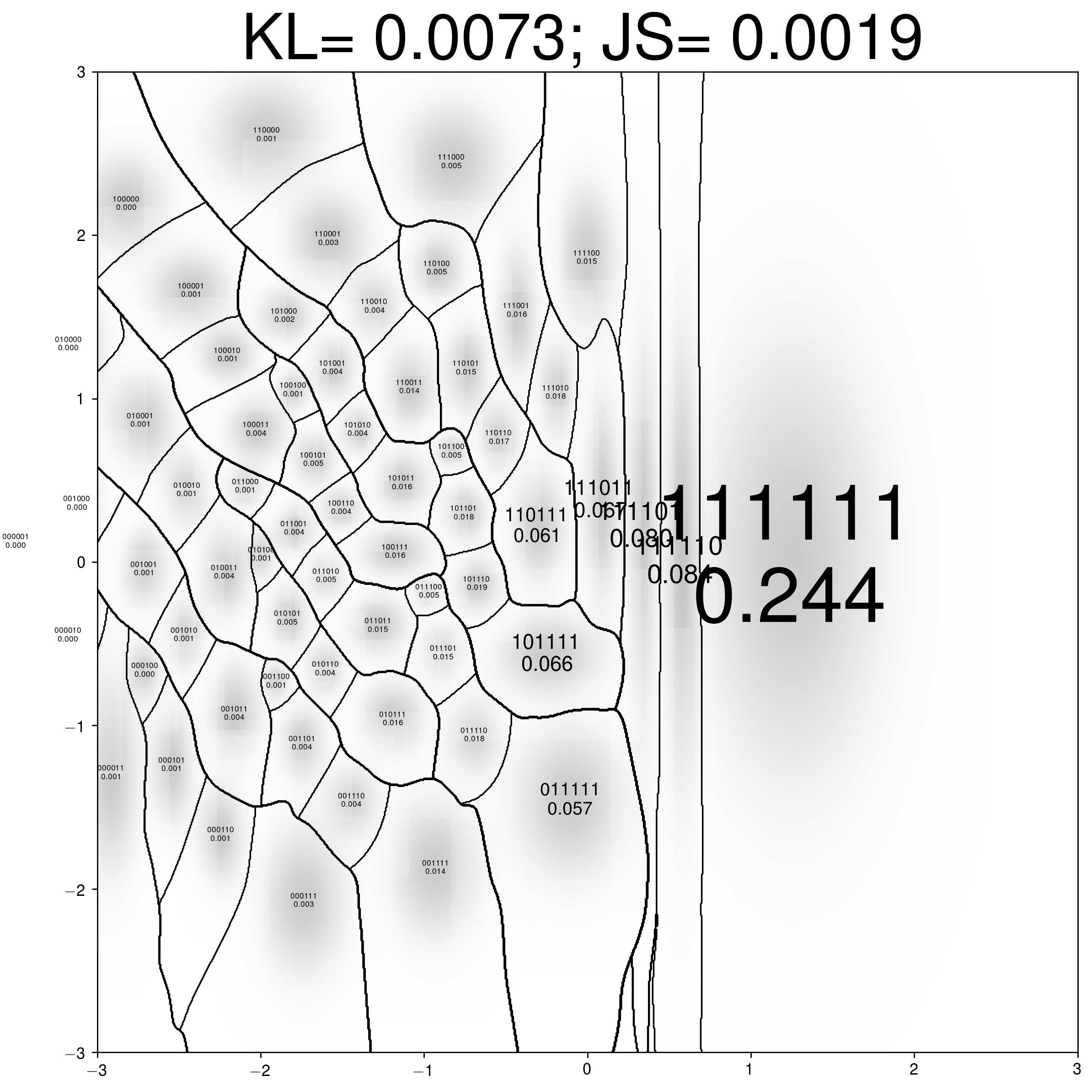} \caption{VAE. G prior; G proposal} \label{fig:1a}
\end{subfigure}
\caption{Learned 2D manifold on synthetic data. Dashed lines indicate proposal boundaries, solid lines indicate decoding boundaries. For each decoded string, we write its probability under deterministic decoding. Left and middle images: DD-VAE; Right image: VAE. U = Uniform, G = Gaussian, T = Tricube.}  \label{fig:synthetic}
\end{center}
\end{figure*}

\section{Experiments}
We experiment on four datasets: synthetic and MNIST datasets to visualize a learned manifold structure, on MOSES molecular dataset to analyze the distribution quality of DD-VAE, and ZINC dataset to see if DD-VAE's latent codes are suitable for goal-directed optimization. We describe model hyperparameters in Appendix~\ref{hyperparameters}. 

\subsection{Synthetic data \label{sec:synthetic}}

This dataset provides a proof of concept comparison of standard VAE with a stochastic decoder and a DD-VAE model with a deterministic decoder. The data consist of 6-bit strings, a probability of each string is given by independent Bernoulli samples with a probability of $1$ being 0.8. For example, a probability of string "110101" is $0.8^4\cdot0.2^2 \approx 0.016$.

In Figure~\ref{fig:synthetic}, we illustrate the 2D latent codes learned with the proposed model. As an encoder and decoder, we used a 2-layer gated recurrent unit (GRU) \citep{cho-etal-2014-learning} network with a hidden size 128. We provide illustrations for a proposed model with a uniform prior and compare uniform and tricube proposals. For a baseline model, we trained a $\beta$-VAE with Gaussian proposal and prior. We used $\beta=0.1$, as for larger $\beta$ we observed posterior collapse. For our model, we used $\beta=1$, which is equivalent to the described model.

For a baseline model, we observe an irregular decision boundary, which also behaves unpredictably for latent codes that are far from the origin. Both uniform and tricube proposals learn a brick-like structure that covers the whole latent space. During training, we observed that the uniform proposal tends to separate proposal distributions by a small margin to ensure there is no overlap between them. As the training continues, the width of proposals grows until they cover the whole space. For the tricube proposal, we observed a similar behavior, although the model tolerates slight overlaps.

\subsection{Binary MNIST}
To evaluate the model on imaging data, we considered a binarized MNIST \citep{lecun-mnisthandwrittendigit-2010} dataset obtained by thresholding the original $0$ to $1$ gray-scale images by a threshold of $0.3$. The goal of this experiment is to visualize how DD-VAE learns 2D latent codes on moderate size datasets.

For this experiment, we trained a 4-layer fully-connected encoder and decoder with structure $784 \to 256 \to 128 \to 32 \to 2$. In Figure~\ref{fig:mnist_latent}, we show learned latent space structure for a baseline VAE with Gaussian prior and proposal and compare it to a DD-VAE with uniform prior and proposal. Note that the uniform representation evenly covers the latent space, as all points have the same prior probability. This property is useful for visualization tasks. The learned structure better separates classes, although it was trained in an unsupervised manner: K-nearest neighbor classifier on 2D latent codes yields $87.8\%$ accuracy for DD-VAE and $86.1\%$ accuracy for VAE.

\begin{figure}[h]
\begin{center}
\includegraphics[width=0.92\columnwidth]{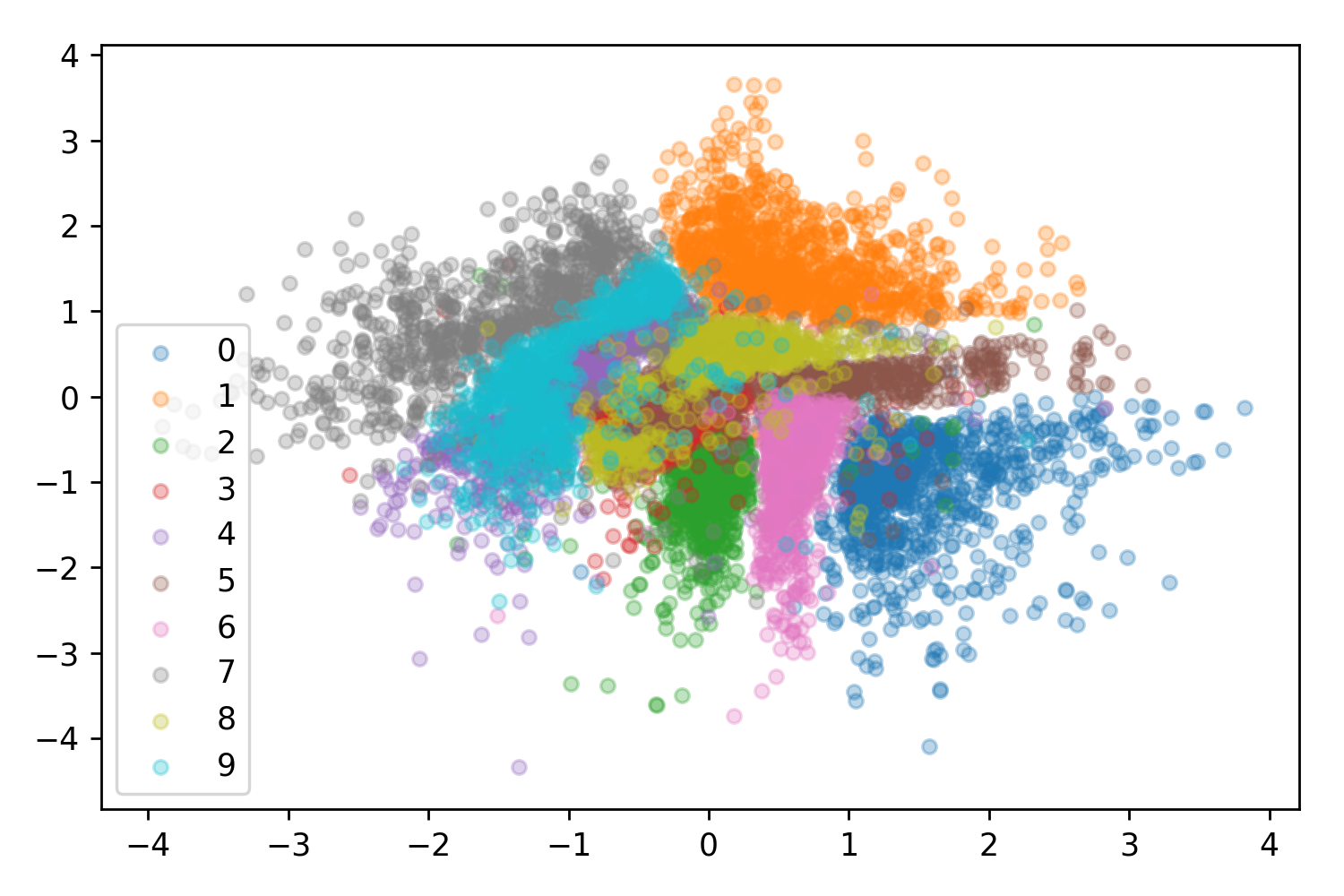}\\\includegraphics[width=0.92\columnwidth]{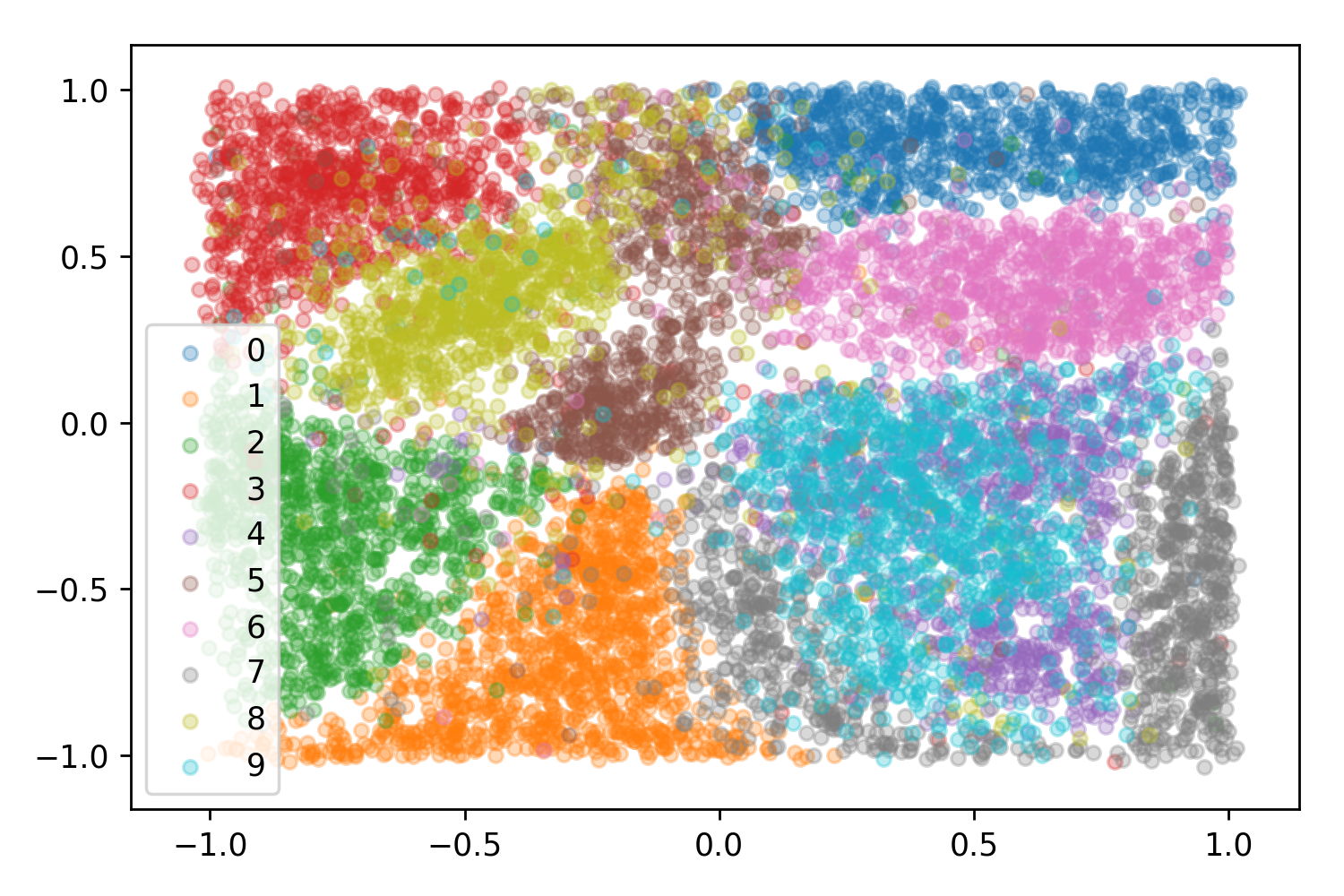}
\caption{Learned 2D manifold on binarized MNIST obtained as proposal means on the test set. {\bf Top:} VAE with Gaussian prior and Gaussian proposal; {\bf Bottom:} DD-VAE with uniform prior and uniform proposal.}  \label{fig:mnist_latent}
\end{center}
\end{figure}

\begin{table*}[h]
\caption{Distribution learning with deterministic decoding on MOSES dataset. We report generative modeling metrics: FCD/Test (lower is better) and SNN/Test (higher is better). Mean $\pm$ std over multiple runs. G = Gaussian proposal, T = Triweight proposal.}
\label{tab:moses_distribution}
\begin{center}
\begin{small}
\begin{sc}
\begin{tabular}{lcccccc}
\toprule
\multirow{2}{*}{Method} & \multicolumn{3}{c}{FCD/Test ($\downarrow$)} & \multicolumn{3}{c}{SNN/Test ($\uparrow$)} \\
\cmidrule(lr){2-4}\cmidrule(l){5-7}
& 70\% & 80\% & 90\% & 70\% & 80\% & 90\% \\
\midrule
VAE (G) & 0.205 {\tiny $\pm$ 0.005} & 0.344 {\tiny $\pm$ 0.003} & 0.772 {\tiny $\pm$ 0.007} & 0.550 {\tiny $\pm$ 0.001}& 0.525 {\tiny $\pm$ 0.001}& 0.488 {\tiny $\pm$ 0.001}\\
VAE (T)  & 0.207 {\tiny $\pm$ 0.004} & 0.335 {\tiny $\pm$ 0.005} & 0.753 {\tiny $\pm$ 0.019} & 0.550 {\tiny $\pm$ 0.001} & 0.526 {\tiny $\pm$ 0.001} & 0.490 {\tiny $\pm$ 0.000} \\
DD-VAE (G) & 0.198 {\tiny $\pm$ 0.012} & 0.312 {\tiny $\pm$ 0.011} & 0.711 {\tiny $\pm$ 0.020} & {\bf 0.555 {\tiny $\pm$ 0.001}} & 0.531 {\tiny $\pm$ 0.001}& 0.494 {\tiny $\pm$ 0.001} \\
DD-VAE (T) & {\bf 0.194 {\tiny $\pm$ 0.001}} & {\bf 0.311 {\tiny $\pm$ 0.010}} & {\bf 0.690 {\tiny $\pm$ 0.010}} & {\bf 0.555 {\tiny $\pm$ 0.000}} & {\bf 0.532 {\tiny $\pm$ 0.001}} & {\bf 0.495 {\tiny $\pm$ 0.001} }\\
\bottomrule
\end{tabular}
\end{sc}
\end{small}
\end{center}
\end{table*}

\subsection{Molecular sets (MOSES)}
In this section, we compare the models on a distribution learning task on MOSES dataset \citep{polykovskiy2018molecular}. MOSES dataset contains approximately $2$ million molecular structures represented as SMILES strings \citep{Weininger1970, Weininger1989}; MOSES also implements multiple metrics, including Similarity to Nearest Neighbor (SNN/Test) and Fr\'echet ChemNet Distance (FCD/Test) \citep{Preuer2018-pf}. SNN/Test is an average Tanimoto similarity of generated molecules to the closest molecule from the test set. Hence, SNN acts as precision and is high if generated molecules lie on the test set's manifold. FCD/Test computes Fr\'echet distance between activations of a penultimate layer of ChemNet for generated and test sets. Lower FCD/Test indicates a closer match of generated and test distributions.

In this experiment, we monitor the model's behavior for high reconstruction accuracy. We trained a $2$-layer GRU encoder and decoder with $512$ neurons and a latent dimension $64$ for both VAE and DD-VAE. We pretrained the models with such $\beta$ that the sequence-wise reconstruction accuracy was approximately $95\%$. We monitored FCD/Test and SNN/Test metrics while gradually increasing $\beta$ until sequence-wise reconstruction accuracy dropped below $70\%$.

In the results reported in Table~\ref{tab:moses_distribution}, DD-VAE outperforms VAE on both metrics. Bounded support proposals have less impact on the target metrics, although they slightly improve both FCD/Test and SNN/Test.

\subsection{Bayesian Optimization}
\begin{table*}[ht]
\caption{Reconstruction accuracy (sequence-wise) and validity of samples on ZINC dataset; Predictive performance of sparse Gaussian processes on ZINC dataset: Log-likelihood (LL) and Root-mean-squared error (RMSE); Scores of top 3 molecules found with Bayesian Optimization. G = Gaussian proposal, T = Tricube proposal.}
\label{tab:rec_val}
\begin{center}
\begin{small}
\begin{sc}
\begin{tabular}{lccccccc}
\toprule
Method & Reconstruction & Validity & LL &  RMSE & top1 &  top2 & top3 \\
\midrule
CVAE  & 44.6\% & 0.7\% & -1.812 $\pm$ 0.004 & 1.504 $\pm$ 0.006 & 1.98 & 1.42 & 1.19 \\
GVAE  & 53.7\% & 7.2\% & -1.739 $\pm$ 0.004 & 1.404 $\pm$ 0.006 & 2.94 & 2.89 & 2.80 \\
SD-VAE  & 76.2\% & 43.5\% & -1.697 $\pm$ 0.015 & 1.366 $\pm$ 0.023 & 4.04 & 3.50 & 2.96 \\
JT-VAE  & 76.7\% & 100.0\% & -1.658 $\pm$ 0.023 & 1.290 $\pm$ 0.026 & 5.30 & 4.93 & 4.49 \\
\midrule
VAE (G) & 87.01\% & 78.32\% & -1.558 $\pm$ 0.019 & 1.273 $\pm$ 0.050 & 5.76 & 5.74 & {\bf 5.67} \\
VAE (T) & 90.3\% & 73.52\% & -1.562 $\pm$ 0.022 & 1.265 $\pm$ 0.051 & 5.41 & 5.38 & 5.35\\
DD-VAE (G) & 89.39\% & 63.07\% & -1.481 $\pm$ 0.020 & 1.199 $\pm$ 0.050  & 5.13 & 4.84 & 4.80 \\
DD-VAE (T) & 89.89\% & 61.38\% & {\bf -1.470 $\pm$ 0.022} & {\bf 1.186 $\pm$ 0.053} & {\bf 5.86} & {\bf 5.77} & 5.64\\
\bottomrule
\end{tabular}
\end{sc}
\end{small}
\end{center}
\end{table*}

\begin{table}[h!t]
\caption{Best molecules found using Bayesian Optimization.}
\label{fig:best_molecules}
\begin{center}
\begin{small}
\begin{sc}
\begin{tabular}{ccc}
\toprule
top1 &  top2 & top3 \\
\midrule
\multicolumn{3}{c}{VAE, Gaussian} \\
\includegraphics[width=60pt]{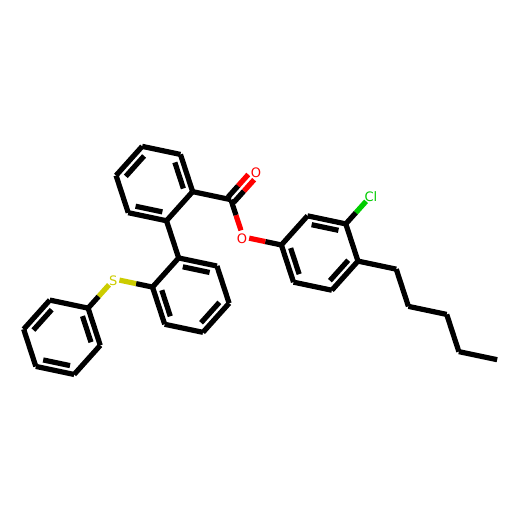} & \includegraphics[width=60pt]{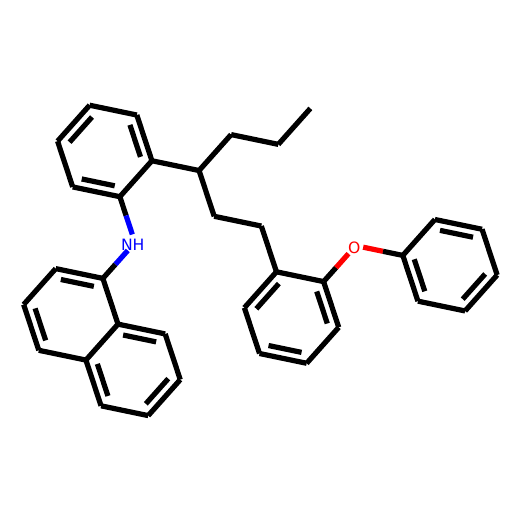} & \includegraphics[width=60pt]{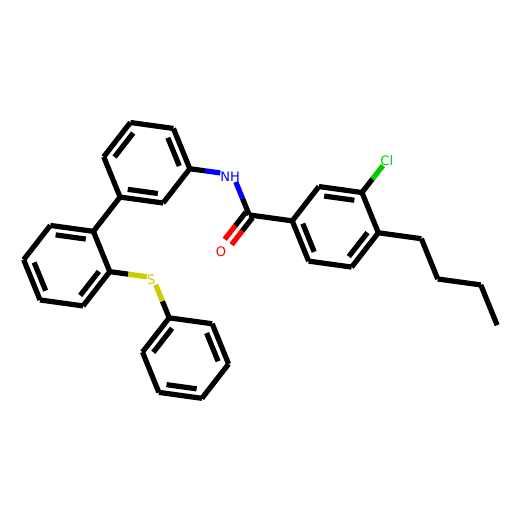}\\
\multicolumn{3}{c}{VAE, Tricube} \\
\includegraphics[width=60pt]{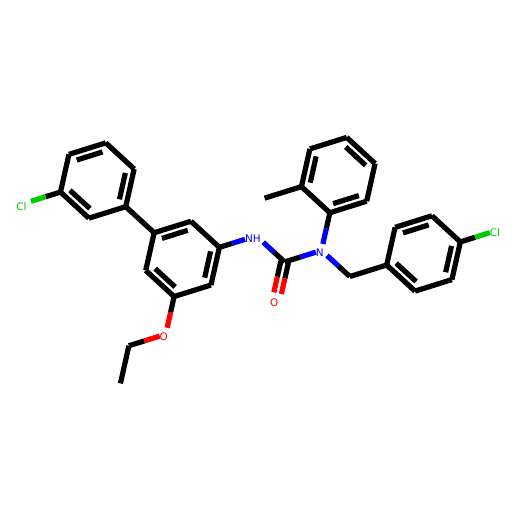} & \includegraphics[width=60pt]{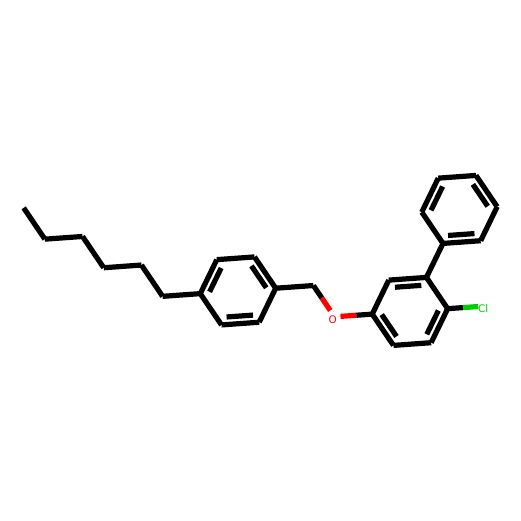} & \includegraphics[width=60pt]{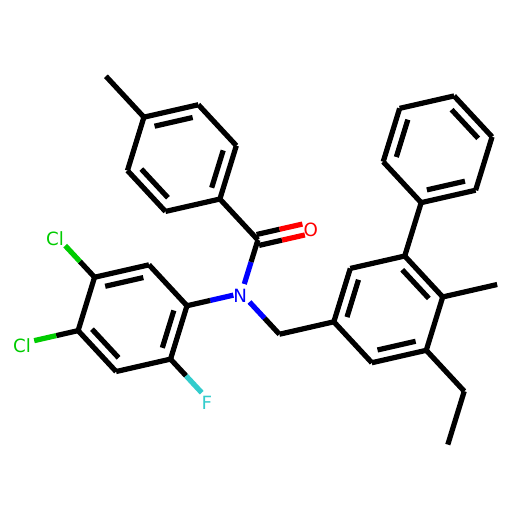}\\
\multicolumn{3}{c}{DD-VAE, Gaussian} \\ \includegraphics[width=60pt]{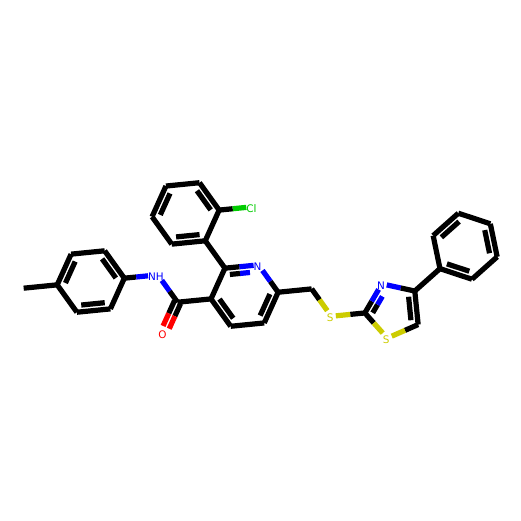} & \includegraphics[width=60pt]{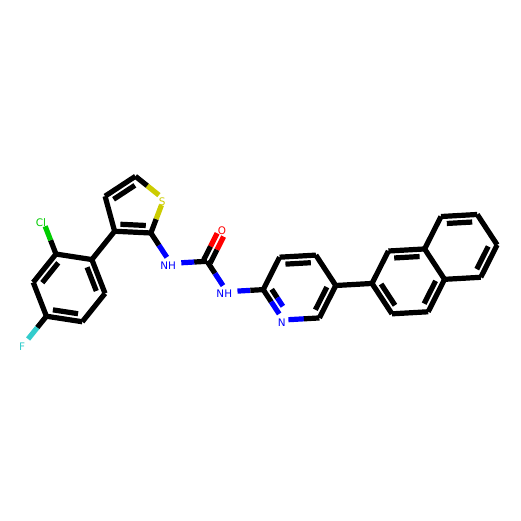} & \includegraphics[width=60pt]{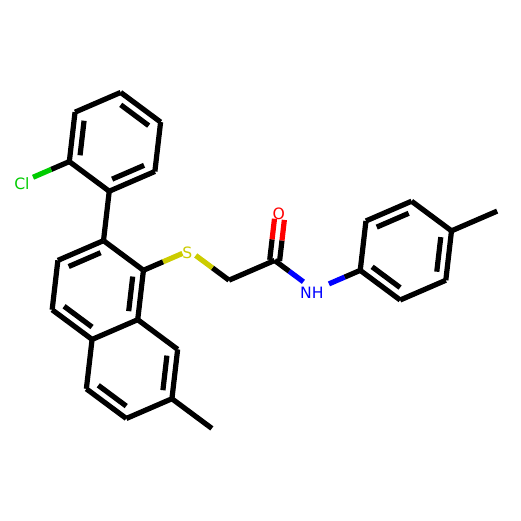}\\
\multicolumn{3}{c}{DD-VAE, Tricube} \\  \includegraphics[width=60pt]{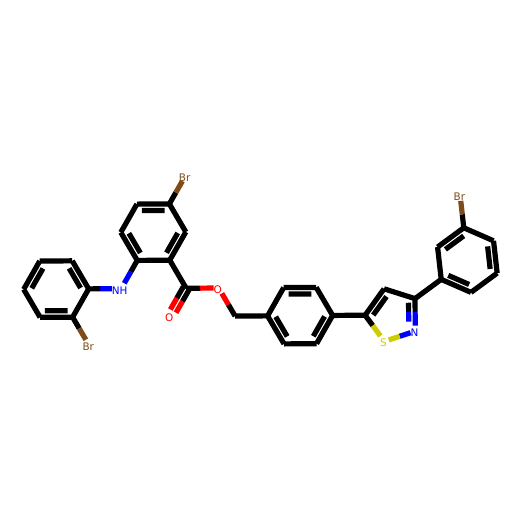} & \includegraphics[width=60pt]{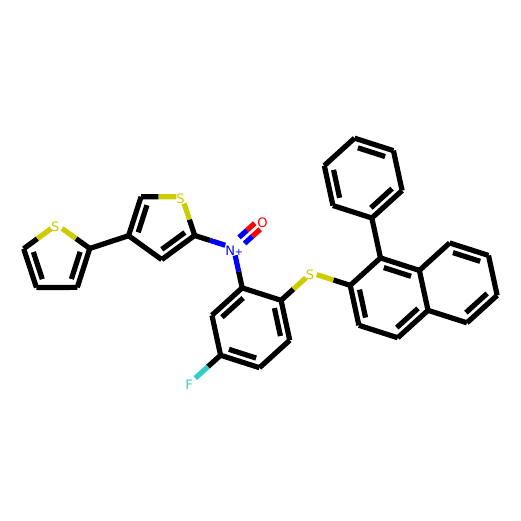} & \includegraphics[width=60pt]{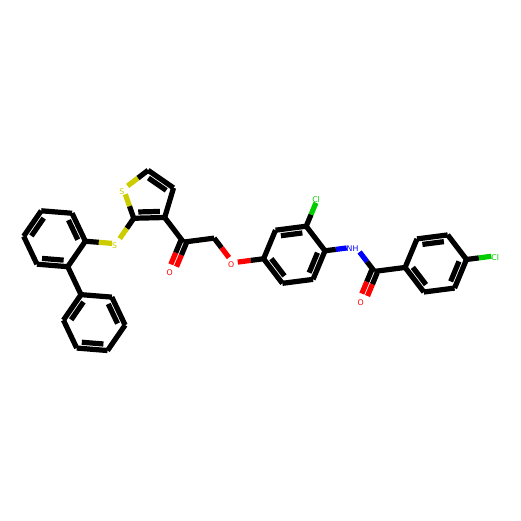}\\
\bottomrule
\end{tabular}
\end{sc}
\end{small}
\end{center}
\end{table}
A standard use case for generative molecular autoencoders for molecules is Bayesian Optimization (BO) of molecular properties on latent codes \citep{gomez2018automatic}. For this experiment, we trained a $1$-layer GRU encoder and decoder with $1024$ neurons on ZINC with latent dimension $64$. We tuned hyperparameters such that the sequence-wise reconstruction accuracy on train set was close to $96\%$ for all our models. The models showed good reconstruction accuracy on test set and good validity of the samples (Table~\ref{tab:rec_val}). We explored the latent space using a standard two-step validation procedure proposed in \citep{kusner2017grammar} to show the advantage of DD-VAE's latent codes. The goal of the Bayesian optimization was to maximize the following score of a molecule $m$:
\begin{equation}
    \textrm{score}(m) = \textrm{logP}(m) - \textrm{SA}(m) - \textrm{cycle}(m),
\end{equation}
where $\textrm{logP}(m)$ is water-octanol partition coefficient of a molecule, $\textrm{SA}(m)$ is a synthetic accessibility score \citep{ertl2009estimation} obtained from RDKit package \citep{landrum2006rdkit}, and $\textrm{cycle}(m)$ penalizes the largest ring $R_{\max}(m)$ in a molecule if it consists of more than 6 atoms:
\begin{equation}
    \textrm{cycle}(m) = \max(0, |R_{\max}(m)| - 6).
\end{equation}
Each component in $\textrm{score}(m)$ is normalized by subtracting mean and dividing by standard deviation estimated on the training set. Validation procedure consists of two steps. First, we train a sparse Gaussian process \citep{snelson2006sparse} on latent codes of DD-VAE trained on approximately $250{,}000$ SMILES strings from ZINC database, and report predictive performance of a Gaussian process on a ten-fold cross validation in Table~\ref{tab:rec_val}. We compare DD-VAE to the following baselines: Character VAE, CVAE \citep{gomez2018automatic}; Grammar VAE, GVAE \citep{kusner2017grammar}; Syntax-Directed VAE, SD-VAE \citep{dai2018syntax}; Junction Tree VAE, JT-VAE \citep{pmlr-v80-jin18a}.

Using a trained sparse Gaussian process, we iteratively sampled $60$ latent codes using expected improvement acquisition function and Kriging Believer Algorithm \citep{cressie1990origins} to select multiple points for the batch. We evaluated selected points and added reconstructed objects to the training set. We repeated training and sampling for 5 iterations and reported molecules with the highest score in Table~\ref{tab:rec_val} and Table~\ref{fig:best_molecules}. We also report top $50$ molecules for our models in Appendix~\ref{sec:best_mols}.

\section{Discussion}
The proposed model outperforms the standard VAE model on multiple downstream tasks, including Bayesian optimization of molecular structures. In the ablation studies, we noticed that models with bounded support show lower validity during sampling. We suggest that it is due to regions of the latent space that are not covered by any proposals: the decoder does not visit these areas during training and can behave unexpectedly there. We found a uniform prior suitable for downstream classification and visualization tasks since latent codes evenly cover the latent space.

DD-VAE introduces an additional hyperparameter $\tau$ that balances reconstruction and $\KL$ terms. Unlike $\KL$ scale $\beta$, temperature $\tau$ changes loss function and its gradients non-linearly. We found it useful to select starting temperatures such that gradients from $\KL$ and reconstruction term have the same scale at the beginning of training. Experimenting with annealing schedules, we found log-linear annealing slightly better than linear annealing.

\subsubsection*{Acknowledgements}
The authors thank Maksim Kuznetsov and Alexander Zhebrak for helpful comments on the paper. Experiments on synthetic data in Section \ref{sec:synthetic} were supported by the Russian Science Foundation grant no.~17-71-20072.

\newpage
\appendix
\section{Proof of Theorem 1 \label{sec:th1 proof}}
We prove the theorem using five lemmas.
\begin{lemma}
$\Lt$ convergences to $\Ls$ pointwise when $\tau$ converges to $0$ from the right:
\begin{equation}
    \forall (\theta, \phi)~~~\lim_{\tau \to 0+} \Lt(\theta, \phi) = \Ls(\theta, \phi)\label{eq:pointwise}
\end{equation}
\end{lemma}
\begin{proof}
To prove Eq.~\ref{eq:pointwise}, we first show that our approximation in Eq.10 from the main paper converges pointwise to $\Ind{x > 0}$. $\forall x \in \mathbb{R}$:
\begin{equation}
    \lim_{\tau \to 0+}\sigma_{\tau}(x) = \lim_{\tau \to 0+}\frac{1}{1 + e^{-x / \tau}\left[\frac{1}{\tau} - 1\right]} = \Ind{x > 0}
\end{equation}
If $x$ is negative, both $e^{-x/\tau}$ and $1/\tau$ converge to $+\infty$, hence $\sigma_{\tau}(x)$ converges to zero. If $x$ is zero, then $\sigma_{\tau}(x) = \tau$ which also converges to zero. Finally, for positive $x$ we apply L'Hôpital's rule to compute the limit:
\begin{equation}
    \lim_{\tau \to 0+} \frac{e^{-x/\tau}}{\tau} = \lim_{\tau \to 0+} \frac{\left(1/\tau\right)'}{\left(e^{x/\tau}\right)'} = \lim_{\tau \to 0+}\frac{e^{-x/\tau}}{x} = 1
\end{equation}
To prove the theorem, we consider two cases. First, if $(\theta, \phi) \notin \Omega$, then for some $x$, $i$, and $x \neq s$,
\begin{equation}
    \mathbb{E}_{z \sim \qq(z \mid x)}\Ind{\widetilde{\pi}^{\theta}_{x, i, x_i}(z) \leq \widetilde{\pi}^{\theta}_{x, i, s}(z)} > 0.
\end{equation}
From the equation above follows that for given parameters the model violates indicators with positive probability. For those $z$, a smoothed indicator function takes values less than $\tau$, so the expectation of its logarithm tends to $-\infty$ when $\tau \to 0+$.

The second case is $(\theta, \phi) \in \Omega$. Since $\Ls(\theta, \phi) > -\infty$, indicators are violated only with probability zero, which will not contribute to the loss neither in $\Ls$, nor in $\Lt$. For all $x$, $i$ and $s$, consider a distribution of a random variable $\delta = \widetilde{\pi}^{\theta}_{x, i, x_i}(z) - \widetilde{\pi}^{\theta}_{x, i, s}(z)$ obtained from a distribution $\qq(z \mid x)$. Let $\delta_{\max} \le 1$ be the maximal value of $\delta$.
We now need to prove that 
\begin{equation}
    \lim_{\tau \to 0+}\mathbb{E}_{\delta \sim p(\delta)}\log \sigma_{\tau}(\delta) = 0
\end{equation}
For any $\epsilon > 0$, we select $\delta_0 > 0$ such that $p(\delta < \delta_0) < \epsilon$. For the next step we will use the fact that $\sigma_{\tau}\left(\delta_{1/2}\right) = 0.5$, where $\delta_{1/2} = \tau\log\left(\frac{1}{\tau}-1\right)$. By selecting $\tau$ small enough such that $\delta_{1/2} < \delta_0$, we split the integration limit for $\delta$ in expectation into three segments: $(0, \delta_{1/2}]$, $(\delta_{1/2}, \delta_0]$, $(\delta_0, \delta_{\max})$. A lower bound on $\log \sigma_{\tau}(\delta)$ in each segment is given by its value in the left end: $\log\tau$, $\log 1/2$, $\log\sigma_{\tau}(\delta_0)$. Also, since $p(\delta \leq 0) = 0$ and $\delta$ is continuous on compact support of $\qq(z \mid x)$, density $p(\delta)$ is bounded by some constant $M$. Such estimation gives us the final lower bound using pointwise convergence of $\sigma_{\tau}(\delta)$:
\begin{equation}
\begin{split}
    0 \ge \mathbb{E}&_{\delta \sim p(\delta)}\log \sigma_{\tau}(\delta) \ge \\
    & M \cdot \underbrace{\log\tau \cdot \delta_{1/2}}_{\lim_{\tau \to 0+} \dots = 0} + \epsilon \cdot \log{1/2} \\ 
    & + M \cdot \underbrace{\log\sigma_{\tau}(\delta_0)}_{\lim_{\tau \to 0+}\dots = 0} \cdot (\delta_{\max} - \delta) \rightarrow_{\tau \to 0+} \epsilon \cdot  \log 1/2.
\end{split}
\end{equation}
We used $\lim_{\tau \to 0+}\log\tau \cdot \delta_{1/2}=0$ which can be proved by applying the L'Hôpital's rule twice.
\end{proof}

\begin{proposition} \label{prop:1}
For our model, $\Ls$ is finite if and only if a sequence-wise reconstruction error rate is zero: 
\begin{equation}
    (\theta, \phi) \in \Omega \Leftrightarrow \Delta(\widetilde{x}_{\theta}, \phi) = 0
\end{equation}
\end{proposition}

\begin{lemma}
Sequence-wise reconstruction error rate $\Delta(\phi)$ is continuous.
\end{lemma}
\begin{proof}
Following equicontinuity in total variation of $\qq(z \mid x)$ at $\phi$ for any $x$ and finiteness of $\chi$, for any $\epsilon > 0$ there exists $\delta > 0$ such that for any $x \in \chi$ and any $\phi'$ such that $\|\phi - \phi'\| < \delta$
\begin{equation}
    \int \left|\qq(z | x) - q_{\phi'}(z | x)\right| dz < \epsilon.
\end{equation}
For parameters $\phi$ and $\phi'$, we estimate the difference in $\Delta$ function values
\begin{equation}
\begin{split}
    \Delta(\phi) & - \Delta(\phi') \\ & = \underbrace{\Delta(\widetilde{x}^*_{\phi}, \phi) - \Delta(\widetilde{x}^*_{\phi'}, \phi)}_{\le 0} + \Delta(\widetilde{x}^*_{\phi'}, \phi) - \Delta(\widetilde{x}^*_{\phi'}, \phi') \\
    & \le \mathbb{E}_{x \sim p(x)}\underbrace{\int\left(\qq(z | x) - q_{\phi'}(z | x)\right)\Ind{\widetilde{x}^*_{\phi'}(z) \neq x}dz}_{< \epsilon}  \\ & \le \epsilon
\end{split}
\end{equation}
Symmetrically, $\Delta(\phi') - \Delta(\phi) \le \epsilon$, resulting in $\Delta(\phi)$ being continuous.
\end{proof}

\begin{lemma}
\label{lemma:3}
Sequence-wise reconstruction error rate $\Delta(\phi_n)$ converges to zero:
\begin{equation}
\lim_{n \to +\infty}\Delta(\phi_n) = \Delta(\widetilde{\phi}) = 0.
\end{equation}
The convergence rate is $\mathcal{O}(\frac{1}{\log(1/\tau_n)})$.
\end{lemma}
\begin{proof}
Since $\Omega$ is not empty, there exists $(\widehat{\theta}, \widehat{\phi}) \in \Omega$. 
From pointwise convergence of $\Lt$ to $\Ls$ at point $(\widehat{\theta}, \widehat{\phi})$, for any $\epsilon > 0$ exists $N$ such that for any $n > N$:
\begin{equation}
    \underbrace{\Lo_{\tau_n}(\theta_n, \phi_n) \ge \Lo_{\tau_n}(\widehat{\theta}, \widehat{\phi})}_{\textrm{from the definition of } (\theta_n, \phi_n)} \ge \Ls(\widehat{\theta}, \widehat{\phi}) - \epsilon. \label{eq:pcw}
\end{equation}
Next, we derive an upper bound on $\Lo_{\tau_n}(\theta_n, \phi_n)$ using the fact that $\log\sigma_{\tau}(x) < 0$ if $x > 0$, and $\log\sigma_{\tau}(x) \le \log\tau_n$ if $x \le 0$:
\begin{equation}
\begin{split}
    \Lo_{\tau_n} (\theta_n, \phi_n) & \le \mathbb{E}_{x \sim p(x)}\bigg[\mathbb{E}_{z \sim \qq(z \mid x)}\sum_{i=1}^{|x|}\sum_{s \neq x_i} \log \tau_n \cdot \\
    & \Ind{\pi_{x, i, x_i}(z) \le \pi_{x, i, s}(z)} \underbrace{- \KL{\qq(z \mid x)}{p(z)}}_{\le 0}\bigg] \\
    & \le |V|L\cdot \log\tau_n \cdot \Delta(\widetilde{x}_{\theta_n}, \phi_n).
\end{split} \label{eq:delta_estimate}
\end{equation}
Combining Eq.~\ref{eq:pcw} and Eq.~\ref{eq:delta_estimate} together we get
\begin{equation}
    |V|L\cdot \underbrace{\log \tau_n}_{< 0} \cdot \Delta(\widetilde{x}_{\theta_n}, \phi_n) \ge \Ls(\theta^*, \phi^*) - \epsilon
\end{equation}
Adding the defintion of $\Delta(\phi)$, we obtain
\begin{equation}
    0 \le \Delta(\phi_n) \le \Delta(\widetilde{x}_{\theta_n}, \phi_n) \le  \frac{\epsilon - \Ls(\theta^*, \phi^*)}{|V|L\cdot\log(1/\tau_n)}
\end{equation}
The right hand side goes to zero when $n$ goes to infinity and hence $\lim_{n \to +\infty}\Delta(\widetilde{x}_{\theta_n}, \phi_n) = 0$ and $\lim_{n \to +\infty}\Delta(\phi_n) = 0$ with the convergence rate $\mathcal{O}(\frac{1}{\log(1/\tau_n)})$. Since $\Delta(\phi_n)$ is continuous, $\Delta(\widetilde{\phi}) = 0$.
\end{proof}

\begin{lemma}
$\Ls(\theta, \phi)$ attains its supremum:
\begin{equation}
    \exists \theta^* \in \Theta, \phi^* \in \Phi: \Ls(\theta^*, \phi^*) = \sup_{\theta \in \Theta, \phi \in \Phi}\Ls(\theta, \phi).
\end{equation}
\end{lemma}
\begin{proof}
From Lemma~\ref{lemma:3}, $\Delta(\widetilde{\phi}) = 0$. Hence, for a choice of $\widetilde{\theta}$ from the theorem statement, $\Delta(\widetilde{\theta}, \widetilde{\phi}) = 0$. Equivalently, $(\widetilde{\theta}, \widetilde{\phi}) \in \Omega$.

Note that since $\Delta(\phi) \ge 0$ is continuous on a compact set, $\Phi_0 = \{\phi \mid \Delta(\phi)=0\}$ is a compact set. Also, $\Ls(\theta, \phi)$ is constant with respect to $\theta$ on $\Omega$. From the theorem statement, for any $\phi$ such that $\Delta(\phi) = 0$, there exists  $\theta(\phi)$ such that $(\theta(\phi), \phi) \in \Omega$. Combining all statements together,
\begin{equation}
\sup_{\phi \in \Phi_0}\Ls(\theta(\phi), \phi) = \sup_{\theta \in \Theta, \phi \in \Phi}\Ls(\theta, \phi)
\end{equation}
In $\Omega$, $\Ls$ is a continuous function: $\forall(\theta, \phi) \in \Omega$,
\begin{equation}
    \Ls(\theta, \phi) = -\KL(\phi) = -\mathbb{E}_{x \sim p(x)}\KLD{\qq(z | x)}{p(z)}
\end{equation}
Hence, continuous function $\Ls(\theta(\phi), \phi)$ attains its supremum on a compact set $\Phi$ at some point $(\theta^*, \phi^*)$, where $\theta^* = \theta(\phi^*)$. 
\end{proof}

\begin{lemma}
Parameters $(\widetilde{\theta}, \widetilde{\phi})$ from theorem statement are optimal:
\begin{equation}
    \Ls(\widetilde{\theta}, \widetilde{\phi}) = \sup_{\theta \in \Theta, \phi \in \Phi}\Ls(\theta, \phi).
\end{equation}
\end{lemma}
\begin{proof}
Assume that $\Ls(\widetilde{\theta}, \widetilde{\phi}) < \Ls(\theta^*, \phi^*)$. Since $(\widetilde{\theta}, \widetilde{\phi}) \in \Omega$ and $(\theta^*, \phi^*) \in \Omega$, $\Ls(\widetilde{\theta}, \widetilde{\phi}) = -\KL(\widetilde{\phi})$ and $\Ls(\theta^*, \phi^*) = -\KL(\phi^*)$. As a result, from our assumption, $\KL(\phi^*) < \KL(\widetilde{\phi})$.

From continuity of $\KL(\phi)$ divergence, for any $\epsilon > 0$, exists $\delta > 0$ such that if $\|\widetilde{\phi} - \phi\| < \delta$,
\begin{equation}
    \KL(\phi) > \KL(\widetilde{\phi}) - \epsilon = \Ls(\widetilde{\theta}, \widetilde{\phi}) - \epsilon
\end{equation}
From the convergence of $\phi_n$ to $\widetilde{\phi}$ and convergence of $\tau_n$ to zero, there exists $N_1$ such that for any $n > N_1$, $\|\widetilde{\phi} - \phi_n\| < \delta$.

From pointwise convergence of $\Lo_{\tau_n}$ at point $(\theta^*, \phi^*)$ to $\Ls(\theta^*, \phi^*)$, for any $\epsilon > 0$, exists $N_2$ such that for all $n > N_2$, $\Lo_{\tau_n}(\theta^*, \phi^*) > \Ls(\theta^*, \phi^*) - \epsilon$. Also, $\Lo_{\tau_n}(\theta_n, \phi_n) \le -\KL(\phi_n)$ from the definition of $\Lo_{\tau_n}$ as a negative $\KL$ divergence plus some non-positive penalty for reconstruction error.

Taking $n > \max(N_1, N_2)$, we get the final chain of inequalities:
\begin{equation}
\begin{split}
    \Lo_{\tau_n}(\theta_n, \phi_n) & \le -\KL(\phi_n) < -\KL(\widetilde{\phi}) + \epsilon \\
    & = \Ls(\widetilde{\theta}, \widetilde{\phi}) + \epsilon < \Lo_{\tau_n}(\theta^*, \phi^*) - \epsilon + \epsilon \\
    & = \Lo_{\tau_n}(\theta^*, \phi^*)
\end{split}
\end{equation}
Hence, $\Lo_{\tau_n}(\theta_n, \phi_n) < \Lo_{\tau_n}(\theta^*, \phi^*)$, which contradicts $(\theta_n, \phi_n) \in \Argmax$ of $\Lo_{\tau_n}$. As a result, $\Ls(\widetilde{\theta}, \widetilde{\phi}) = \Ls(\theta^*, \phi^*)$.
\end{proof}

\section{Implementation details \label{hyperparameters}}
For all experiments, we provide configuration files in a human-readable format in the supplementary code. Here we provide the same information for convenience.

\subsection{Synthetic data}
Encoder and decoder were GRUs with $2$ layers of $128$ neurons. The latent size was $2$; embedding dimension was $8$. We trained the model for $100$ epochs with Adam optimizer with an initial learning rate $5\cdot10^{-3}$, which halved every $20$ epochs. The batch size was 512. We fine-tuned the model for $10$ epochs after training by fixing the encoder and learning only the decoder. For a proposed model with a uniform prior and a uniform proposal, we increased $\KL$ weight $\beta$ linearly from $0$ to $0.1$ during $100$ epochs. For the Gaussian and tricube proposals, we increased $\KL$ weight $\beta$ linearly from $0$ to $1$ during $100$ epochs. For all three experiments, we pretrained the autoencoder for the first two epochs with $\beta=0$. We annealed the temperature from $10^{-1}$ to $10^{-3}$ during $100$ epochs of training in a log-linear scale. For a tricube proposal, we annealed the temperature to $10^{-2}$.

\begin{figure*}[ht]
\begin{center}
\includegraphics[width=1\columnwidth]{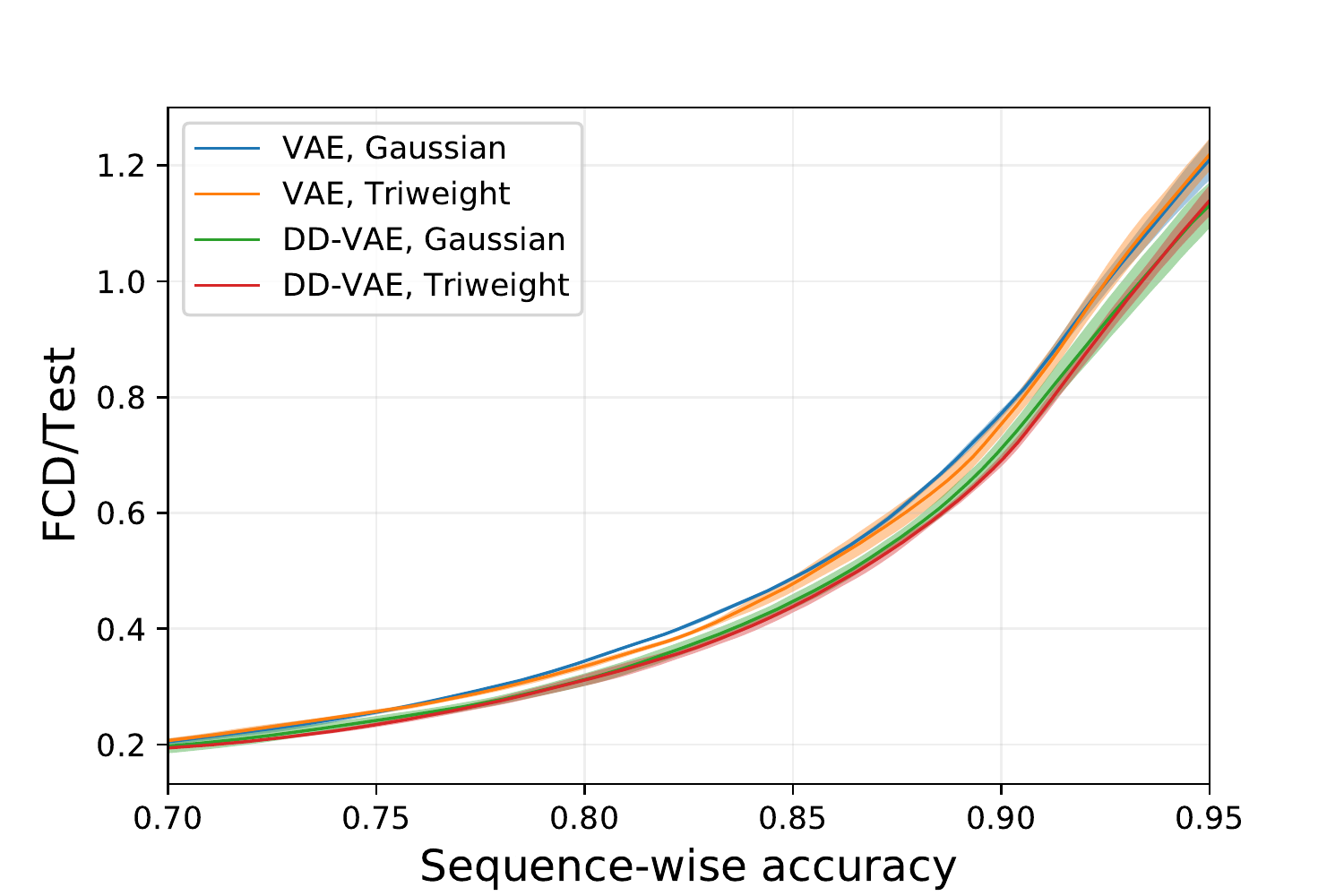}~\includegraphics[width=1\columnwidth]{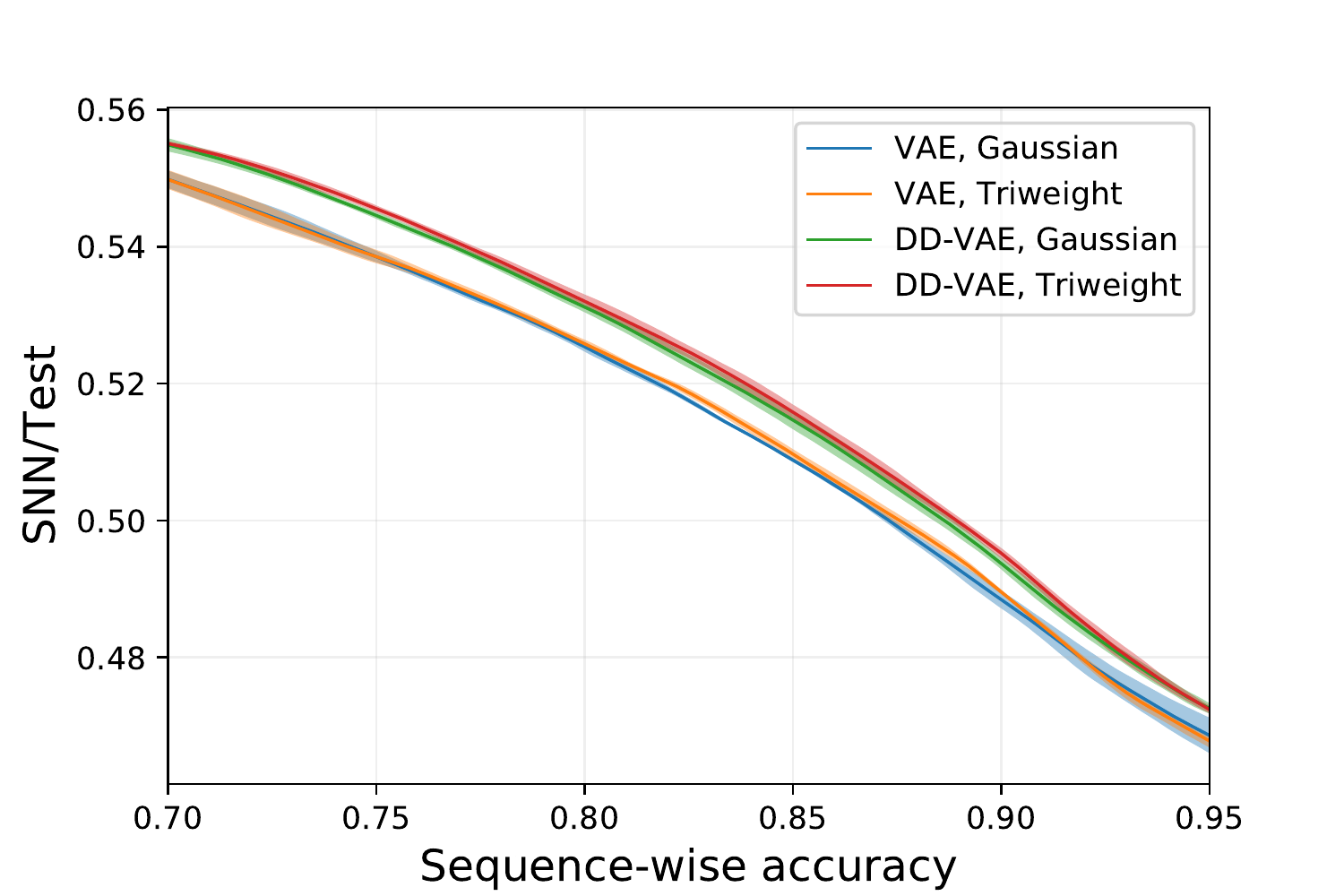}
\caption{Distribution learning with deterministic decoding on MOSES dataset: FCD/Test (lower is better) and SNN/Test (higher is better). Solid line: mean, shades: std over multiple runs.} \label{fig:moses_appendix}
\end{center}
\end{figure*}

\subsection{Binary MNIST}
We binarized the dataset by thresholding original MNIST pixels with a value of $0.3$. We used a fully connected neural network with layer sizes $784 \to 256 \to 128 \to 32 \to 2$ with LeakyReLU activation functions. We trained the model for $150$ epochs with a starting learning rate $5\cdot10^{-3}$ that halved every $20$ epochs. We used a batch size $512$ and clipped the gradient with value $10$. We increased $\beta$ from $10^{-5}$ to $0.005$ for VAE and $0.05$ for DD-VAE. We decreased the temperature in a log scale from $0.01$ to $0.0001$.

\subsection{MOSES}
We used a $2$-layer GRU network with a hidden size of $512$. Embedding size was $64$, the latent space was $64$-dimensional. We used a tricube proposal and a Gaussian prior. We pretrained a model with a fixed $\beta$ for $20$ epochs and then linearly increased $\beta$ for $180$ epochs. We halved the learning rate after pretraining. For DD-VAE models, we decreased the temperature in a log scale from $0.2$ to $0.1$. We linearly increased $\beta$ divergence from $0.0005$ to $0.01$ for VAE models and from $0.0015$ to $0.02$.

\subsection{ZINC}
We used a $1$-layer GRU network with a hidden size of $1024$. Embedding size was $64$, the latent space was $64$-dimensional. We used a tricube proposal and a Gaussian prior. We trained a model for $200$ epochs with a starting learning rate $5\cdot10^{-4}$ that halved every $50$ epochs. We increased divergence weight $\beta$ from $10^{-3}$ to $0.02$ linearly during the first $50$ epochs for DD-VAE models, from $10^{-4}$ to $5\cdot 10^{-4}$ for VAE model, and from $10^{-4}$ to $8\cdot 10^{-4}$ for VAE model with a tricube proposal. We decreased the temperature log-linearly from $10^{-3}$ to $10^{-4}$ during the first $100$ epochs for DD-VAE models. With such parameters we achieved a comparable train sequence-wise reconstruction accuracy of $95\%$.

\section{MOSES distribution learning}
In Figure \ref{fig:moses_appendix}, we report detailed results for the experiment from Section 4.3.

\section{Best molecules found for ZINC \label{sec:best_mols}}
In Figure~\ref{fig:best_molecules_all_DT}, Figure~\ref{fig:best_molecules_all_DG}, Figure~\ref{fig:best_molecules_all_VT}, and Figure~\ref{fig:best_molecules_all_VG} we show the best molecules found with Bayesian optimization during 10-fold cross validation.

\begin{figure*}[t]
\begin{center}
\includegraphics[width=0.8\textwidth]{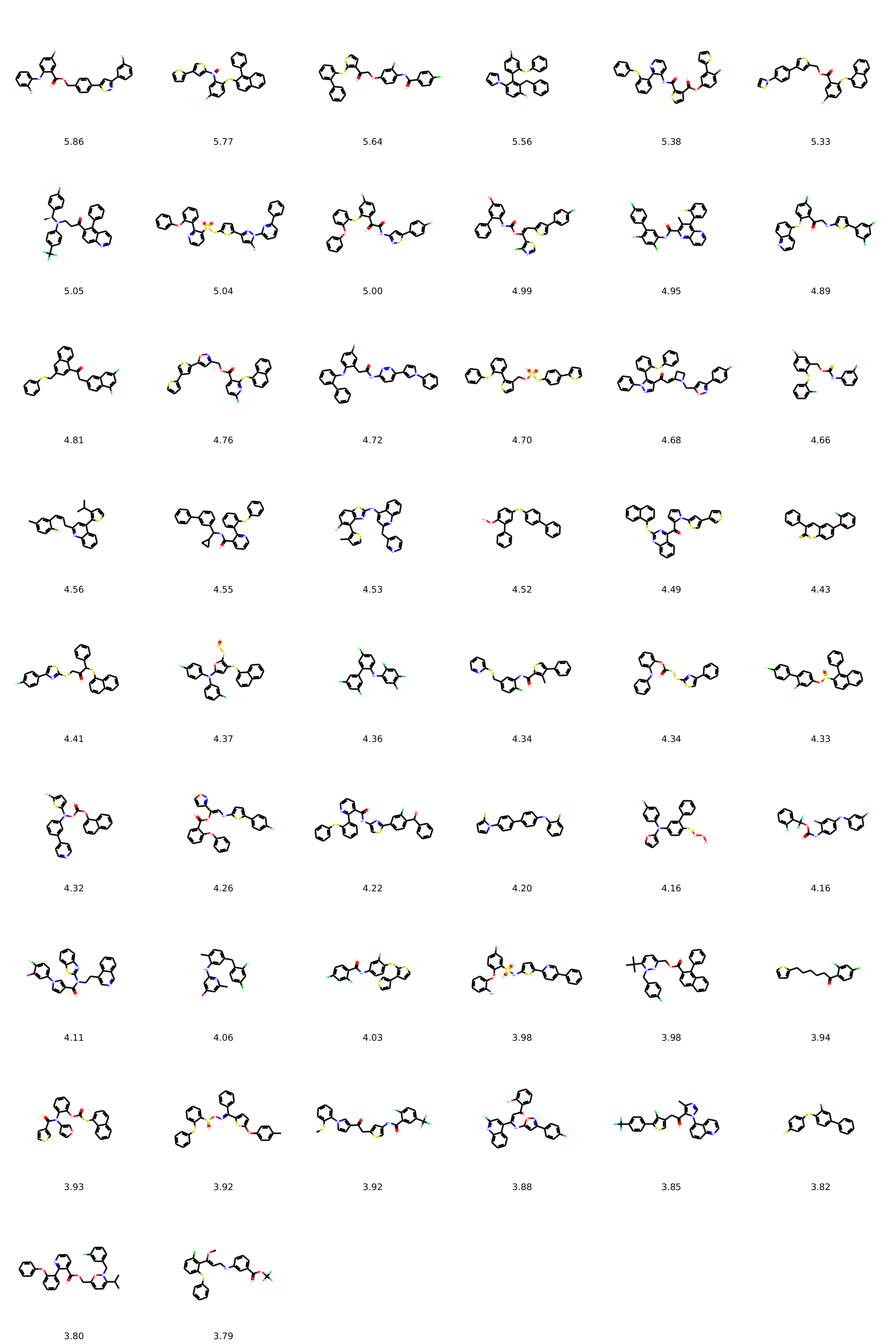}
\caption{DD-VAE with Tricube proposal. The best molecules found with Bayesian optimization during 10-fold cross validation and their scores.} \label{fig:best_molecules_all_DT}
\end{center}
\end{figure*}

\begin{figure*}[t]
\begin{center}
\includegraphics[width=0.8\textwidth]{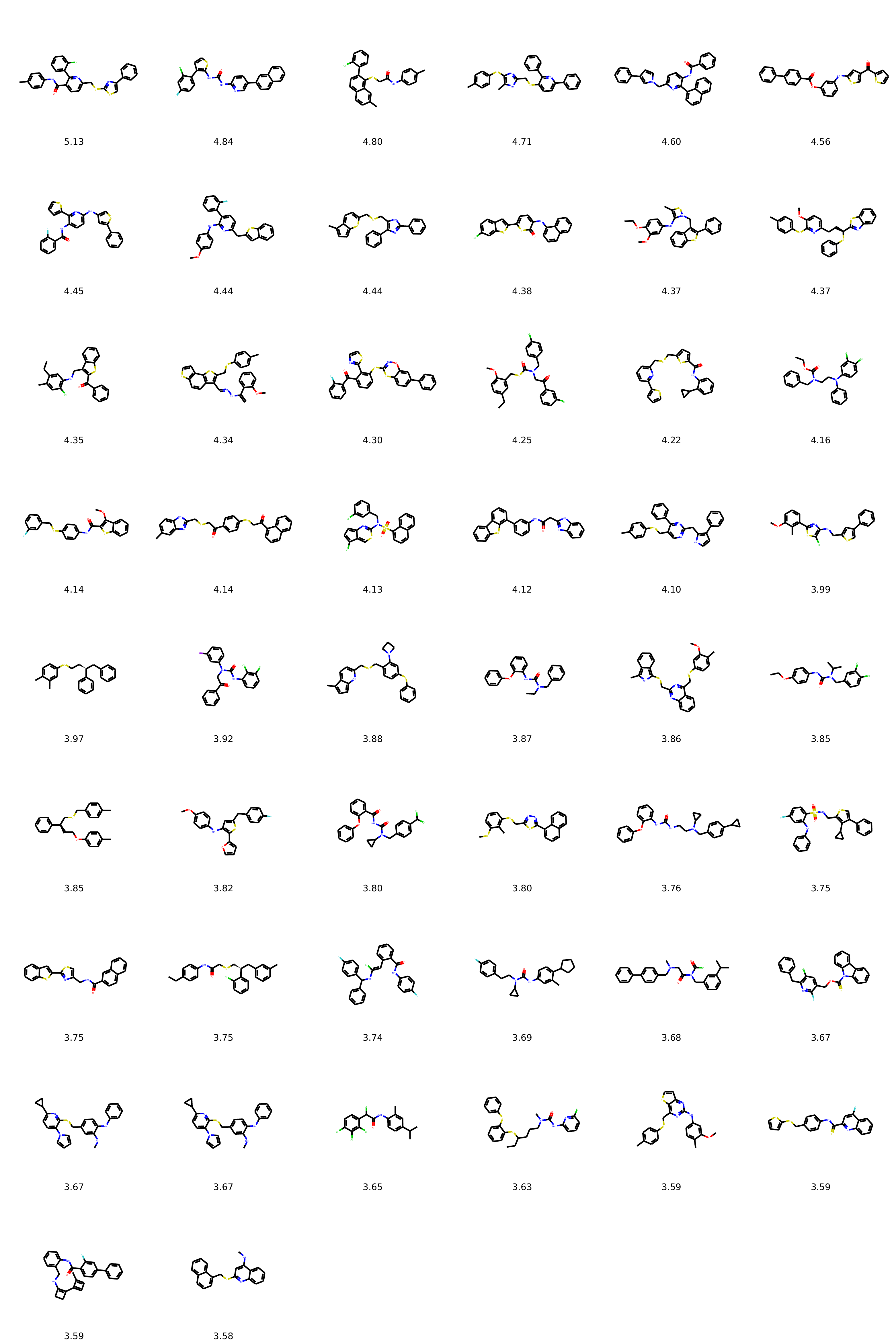}
\caption{DD-VAE with Gaussian proposal. The best molecules found with Bayesian optimization during 10-fold cross validation and their scores.} \label{fig:best_molecules_all_DG}
\end{center}
\end{figure*}

\begin{figure*}[t]
\begin{center}
\includegraphics[width=0.8\textwidth]{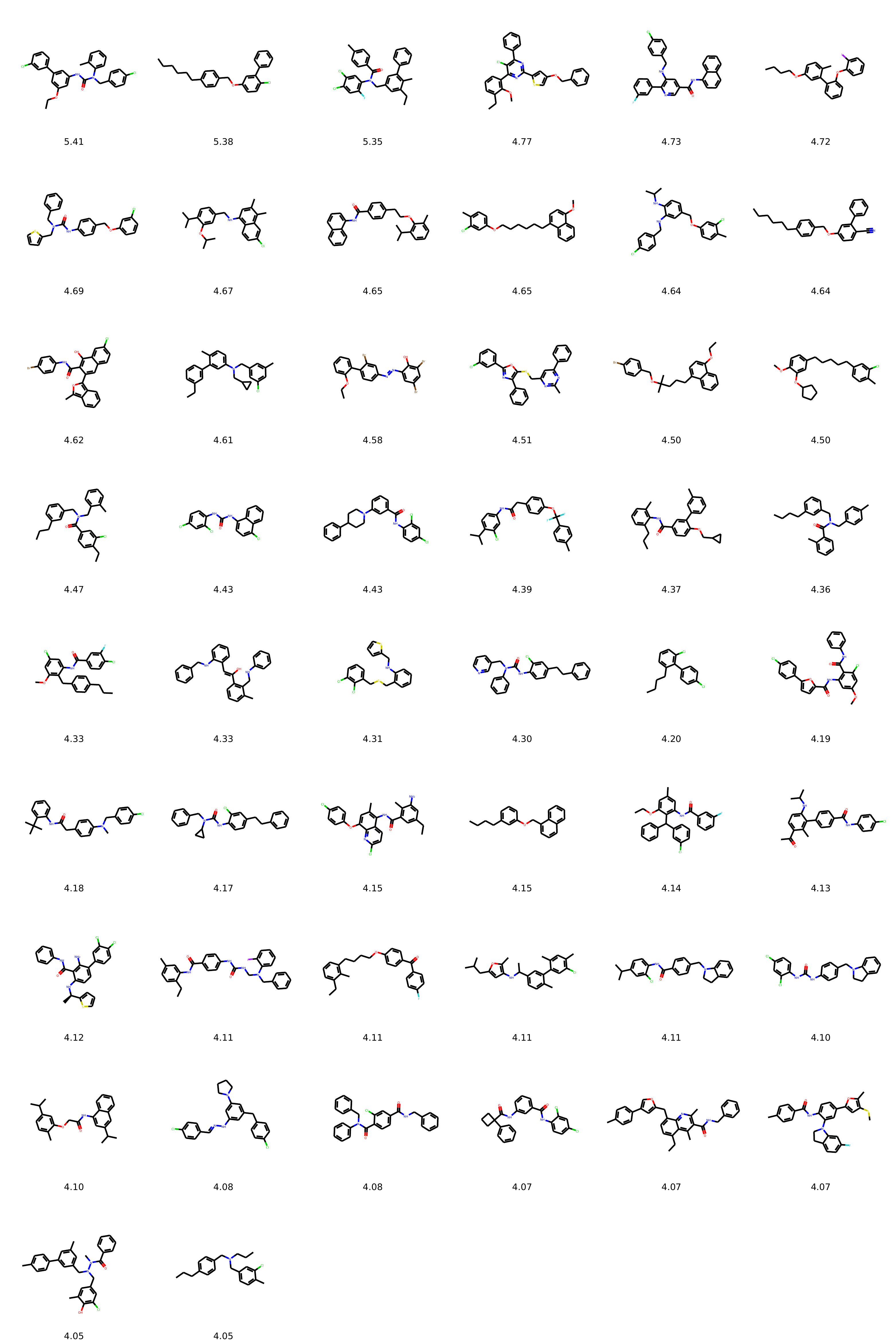}
\caption{VAE with Tricube proposal. The best molecules found with Bayesian optimization during 10-fold cross validation and their scores.} \label{fig:best_molecules_all_VT}
\end{center}
\end{figure*}

\begin{figure*}[t]
\begin{center}
\includegraphics[width=0.8\textwidth]{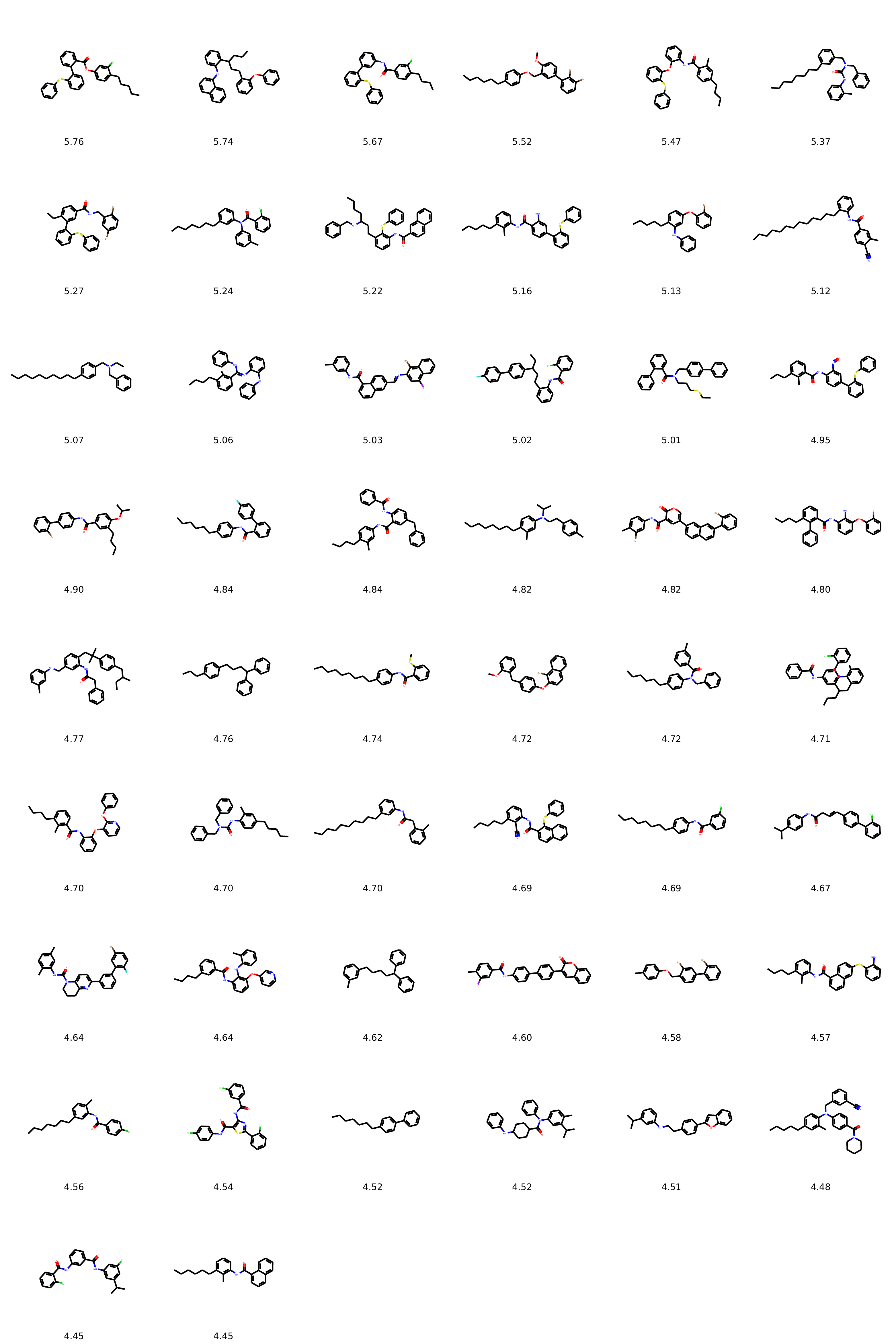}
\caption{VAE with Gaussian proposal. The best molecules found with Bayesian optimization during 10-fold cross validation and their scores.} \label{fig:best_molecules_all_VG}
\end{center}
\end{figure*}

\end{document}